\documentclass{article} 
\usepackage{iclr2024_conference,times}

\usepackage{amsmath,amsfonts,bm}

\def\1{\bm{1}}

\DeclareMathAlphabet{\mathsfit}{\encodingdefault}{\sfdefault}{m}{sl}
\SetMathAlphabet{\mathsfit}{bold}{\encodingdefault}{\sfdefault}{bx}{n}

\DeclareMathOperator*{\argmax}{arg\,max}

\usepackage{hyperref}
\usepackage{url}

\usepackage[utf8]{inputenc} 
\usepackage[T1]{fontenc}    
\usepackage{hyperref}       
\usepackage{url}            
\usepackage{booktabs}       
\usepackage{amsfonts}       
\usepackage{nicefrac}       
\usepackage{microtype}      
\usepackage{xcolor}         
\usepackage{setspace}
\usepackage{amsmath}
\usepackage{amssymb} 
\usepackage{amsthm}
\usepackage{amsfonts, mathtools}
\usepackage{algorithm,algpseudocode}
\usepackage{bm}
\usepackage{bbm}
\usepackage{comment}
\usepackage{graphicx}
\usepackage{multirow, booktabs}
\usepackage{caption}
\usepackage{subcaption}

\usepackage{siunitx} 

\newcommand{\p}{ {p}}

\newcommand{\HH}{ \mathtt{H}}
\newcommand{\MM}{ \mathtt{M}}

\newtheorem{proposition}{Proposition}
\newtheorem{lemma}{Lemma}

\newtheorem{theorem}{Theorem}
\newtheorem{corollary}{Corollary}

\iclrfinalcopy

\title{Learning to Make Adherence-Aware Advice}

\author{Guanting Chen$^{1}$,\,\,\, Xiaocheng Li$^{2}$,\,\,\, Chunlin Sun$^{3}$,\,\,\, Hanzhao Wang$^{2}$  \\
$^{1}$\,\,Department of Statistics and Operations Research, UNC-Chapel Hill\\
$^{2}$\,\,Imperial College Business School, Imperial College London\\
$^{3}$\,\,Institute for Computational and Mathematical Engineering, Stanford University\\
\texttt{guanting@unc.edu} \\
\texttt{\{xiaocheng.li, h.wang19\}@imperial.ac.uk} \\
\texttt{chunlin@stanford.edu} \\
}

\begin{document}

\maketitle

\begin{abstract}
As artificial intelligence (AI) systems play an increasingly prominent role in human decision-making, challenges surface in the realm of human-AI interactions. One challenge arises from the suboptimal AI policies due to the inadequate consideration of humans disregarding AI recommendations, as well as the need for AI to provide advice selectively when it is most pertinent. This paper presents a sequential decision-making model that (i) takes into account the human's adherence level (the probability that the human follows/rejects machine advice) and (ii) incorporates a defer option so that the machine can temporarily refrain from making advice. We provide learning algorithms that learn the optimal advice policy and make advice only at critical time stamps. Compared to problem-agnostic reinforcement learning algorithms, our specialized learning algorithms not only enjoy better theoretical convergence properties but also show strong empirical performance.
\end{abstract}

\section{Introduction}
Artificial intelligence (AI) has achieved remarkable success across various aspects of everyday life. However, it is crucial to acknowledge that many of AI's accomplishments have been developed as fully automatic systems \citep{mnih2015human,silver2017mastering}. In several important domains like AI-assisted driving \citep{balachandran2021human} and AI-assisted healthcare \citep{shaheen2021applications}, AI is faced with the challenge of interacting with humans \citep{mozannar2020consistent,de2021classification}, introducing a more intricate and demanding dynamic. This interaction between AI and humans gives rise to two significant issues. Firstly, it is common for humans to reject following AI's advice, and if AI assumes humans' perfect adherence to its advice, the advice generated under this assumption may not be optimal. Secondly, humans may prefer AI to refrain from constant advice-giving, opting for AI intervention only when necessary. They may value their autonomy when performing well but expect AI guidance during critical moments or when they encounter situations in which they are typically less proficient. These considerations underscore the importance of comprehending human behavior and preferences to develop effective and adaptable AI systems for human-AI interactions. 

To address the mentioned challenges, in this paper, we provide a decision-making model for human-AI interactions. For the first challenge, the model takes into account the human's \textbf{adherence level}, defined as the probability that the human takes the AI's advice. This allows the machine to account for variations in human adherence level when making advice. For the second challenge, the AI model features an action named \textbf{defer}, which refrains from giving advice to humans. This feature recognizes that there are instances when humans prefer autonomy and only seek AI guidance during critical moments or situations where they typically struggle. By integrating the adherence level and action deferral into our model, we formulate these challenges as a decision-making problem.

To cater to this specialized decision-making model, we have developed tailored learning algorithms that are both provably convergent and empirically efficient. These algorithms are specifically designed to effectively handle the unique characteristics and challenges of the human-AI interaction setting.

\subsection{Related Work}
\noindent\textbf{Human-AI interactions}. \hspace{3mm} Human-AI interactions have long been studied in fields such as robotics. Methods for modeling human behaviors and collaborating with robots \citep{bobu2020less,laidlaw2022boltzmann,carroll2019utility} have achieved strong empirical performance. Similar to our definition of adherence level, a stream of literature \citep{chen2018planning, williams2023computational} integrates trust \citep{khavas2020modeling} as latent factors into the human-AI model and solves Partially Observable Markov Decision Process (POMDP) to get policies with strong empirical outcomes. Our work primarily centers on modeling and establishing theoretical foundations for the human-AI interaction model and the associated learning problems, thereby complementing the existing body of human-AI interaction literature.

\noindent\textbf{Modeling human-AI interactions}. \hspace{3mm} On the modeling side, \cite{grand2022best} propose the decision-making model that incorporates the adherence level and illustrates that when the adherence level is low, the optimal advice can be different from the optimal decision. Also, see \cite{sun2022predicting} for an applied setting of interacting with different adherence levels, \cite{shani2019exploration} for the relationship between the model and the exploration-conscious RL setting, and \cite{jacq2022lazy} for the so-called lazy-MDP that features an action similar to defer in our setting. 

\noindent\textbf{Machine learning in human-AI interactions}. \hspace{3mm} Although there has been no literature associated with learning the decision-making model similar to \cite{grand2022best} and \cite{jacq2022lazy}, other machine learning approaches have been put forward \citep{bastani2021improving, meresht2020learning, straitouri2021reinforcement,okati2021differentiable,chen2022algorithmic,hong2023learning,mao2024two, mohri2023learning} with different human-AI interaction settings.

\noindent\textbf{Theoretical reinforcement learning}. \hspace{3mm} Our first proposed algorithm is an optimism-based reinforcement learning method that learns the optimal advice policy. This approach is inspired by the theoretical online reinforcement learning literature \citep{jaksch2010near, lattimore2014near, dann2015sample, azar2017minimax, dann2017unifying, zanette2019tighter, domingues2021episodic}. Instead of directly applying the upper confidence bound in the literature, we customize the learning algorithm so that it leverages special properties in our decision-making model, resulting in advantages in theoretical properties and empirical performance. Our second algorithm adopts a reward-free exploration (RFE) approach \citep{jin2020reward}, which first explores the environment for a given number of episodes, and then becomes capable of outputting near-optimal policy for any bounded reward functions. We find this approach works well for learning algorithms that make pertinent advice. See \cite{zhang2020task, kaufmann2021adaptive, menard2021fast, miryoosefi2022simple} for the follow-up works in RFE.

Our contribution is twofold:

First, we propose a decision-making model for advice-giving that incorporates human's adherence level and an option for the AI to defer the advice and trust the human. This is a comprehensive modeling framework for effective human-AI interactions, where the optimal decision-making not only considers human adherence level but also makes advice/recommendations only at critical states.

Second, based on this decision-making model, we develop tailored learning algorithms that output near-optimal advice policies and know when to make pertinent advice. Compared to the state-of-the-art problem-agnostic RL algorithms, our algorithm features tighter sample complexity bound and stronger empirical performance.

\section{Model Setup}\label{sec_model}
Consider a human decision-maker that takes sequential actions under an episodic Markov decision process (MDP) described by the tuple $\mathcal{M}^{\HH} = (\mathcal{S}, \mathcal{A}, H, p, r)$.  The superscript ${\HH}$ emphasizes the human's involvement in this MDP, $\mathcal{S}$ denotes the set of states, $\mathcal{A}$ denotes the set of actions, $H$ is the horizon of each episode (different from the superscript ${\HH}$), $p$ denotes a deterministic time-dependent transition kernel so that $p_h(s'|s , a)$ is the transition probability from state $s\in \mathcal{S}$ to state $s'\in \mathcal{S}$ under the action $a\in\mathcal{A}$ at time $h$, and $r$ denotes a time-dependent reward function where $r_h(s, a)\in[0,1]$. Let $S = |\mathcal{S}|$ and $A = |\mathcal{A}|$ denote the cardinality of $\mathcal{S}$ and $\mathcal{A}$, respectively.

Suppose the human follows a fixed (suboptimal) policy $\pi^{\HH}$ such that the probability of taking action $a$ at state $s$ and time $h$ is $\pi_h^{\HH}(a|s)$. Alongside the human, an intelligent machine makes advice as decision support to improve the reward collected under $\pi_h^{\HH}$. In other words, the machine does not seek to change human policy but rather improve its final outcome given its suboptimality. Specifically, upon the arrival at each state, the machine can choose to make advice $a^{\MM}\in\mathcal{A}$ to the human (the superscript $\MM$ stands for the machine), or to trust the human and defer the action to the human, denoted by $a^{\MM}=\text{defer}$. If the machine chooses to defer, the human follows its default policy $\pi^{\HH}.$ If the machine chooses to advise, the human takes the machine's advice with probability $\theta(s,a^{\MM})\in [0,1]$, where $\theta(\cdot,\cdot)$ is the \textit{adherence} level of the human, and is defined as follows.

\textbf{Definition 1} \textit{The human's adherence level $\theta :  \mathcal{S} \times \mathcal{A} \rightarrow [0,1]$ is the probability of human adopting/adhering to the machine's certain advice at a certain state.}

Given the setup, the human takes action $a^{\HH}$ according to the following law:
\begin{equation}
\label{eqn:p_act}
    \begin{aligned}
    \mathbb{P}_h(a^{\mathtt{H}} = a|s, a^{\MM}) = \begin{cases}
    \pi_h^{\HH}(a|s), & \text{if $a^{\MM}=\text{defer}$},\\
     \theta(s,a^{\MM}),  &  \text{ if $a^{\MM}\neq\text{defer}$ and $a = a^{\MM}$}  \text{ (adhere),}\\
    (1 - \theta(s,a^{\MM}))\cdot\cfrac{\pi_h^{\HH}(a|s)}{1 - \pi_h^{\HH}(a^\MM|s)}, & \text{  if $a^{\MM}\neq\text{defer}$ and $a \neq a^{\MM}$} \text{ (not adhere)}.
    \end{cases}
    \end{aligned}
\end{equation}
To summarize, under the human-machine interaction, the underlying dynamic becomes
$$s_h \xrightarrow{\text{machine makes advice}} a^{\MM} \xrightarrow{\text{$a^{\HH}\sim\mathbb{P}_h(\cdot|s_h,a^{\MM})$}} a^{\HH} \xrightarrow{\text{$s_{h+1} \sim p_h(\cdot|s_h,a^{\HH})$}} s_{h+1}.$$
At each time $h$, the machine first makes the advice $a^{\MM}$ upon the state $s_h$ and the human incorporates the machine advice into a final action $a^{\HH}$, and then transit to the next state $s_{h+1}$.

\noindent\textbf{The machine's MDP}. From the machine's perspective, the MDP is slightly different from the MDP faced by human. It can be described by $\mathcal{M}^{\MM} = \left(\mathcal{S}, \bar{\mathcal{A}}, H, p^{\MM}, r^{\MM}\right)$. This MDP shares the same state space $\mathcal{S}$ and horizon $H$ as the human MDP $\mathcal{M}^{\HH}$. The action space is augmented to include the defer option $\bar{\mathcal{A}} = \mathcal{A} \cup \{\text{defer}\}$. In the machine's perspective, the transition  can be viewed as a direct consequence of making advice $a^{\MM}\in\bar{\mathcal{A}}$ (i.e, $s_h \to a^{\MM} \to s_{h+1}$), and the transition kernel becomes
\begin{equation}\label{eqn:pandr}
    \begin{aligned}
        p_h^{\MM}(s'|s,a^{\MM}) &= \sum_{a^{\HH} \in \mathcal{A}}p_h(s'|s,a^{\HH})\cdot \mathbb{P}_h(a^{\HH}|s,a^{\MM}),\\
    \end{aligned}
\end{equation}
where $p_h$ is the transition kernel of the MDP $\mathcal{M}^{\HH}$, and the probability $\mathbb{P}_h(\cdot|s,a)$ is specified by the adherence dynamics \eqref{eqn:p_act}. In parallel, we define the reward by marginalizing human's action
\begin{equation*}
r_{h}^{\MM}(s, a^{\MM}) = \sum_{a^{\HH} \in \mathcal{A}}r_h(s,a^{\HH})\cdot \mathbb{P}_h(a^{\HH}|s,a^{\MM}).
\end{equation*}

Denote $\bm{\pi} = \{\pi_h\}_{h\in[H]}$ the machine's policy where $\pi_h: \mathcal{S} \to \bar{\mathcal{A}}$. The value function then becomes
$$V_{h_0}^{\pi}(s) = \mathbb{E}\left[\sum_{h=h_0}^{H}r^{\MM}_{h}(s_{h}, a_{h})\Big \vert s_{h_0} = s\right], \,\,\, \text{where $a_{h} = \pi_{h}(s_{h})$ and $s_{h+1}\sim p_{h}^{\MM}(\cdot | s_{h}, a_{h})$},$$
and let $V^{\pi}_{H+1}(s) = 0$ for any $s \in \mathcal{S}$. 
The optimal value $V^*$ and the optimal policy $\pi^*$ are defined by
$$V_h^*(s) = \max_{\pi\in \Pi} V_h^{\pi}(s),  \ \ \pi^* = \argmax_{\pi\in \Pi} V_1^{\pi}(s)$$
where $\Pi$ consists of all deterministic non-anticipating Markov policies. Similarly, we define the corresponding $Q$ functions to be
$$Q_h^{\pi}(s,a) = r^{\MM}_h(s,a) + \sum_{s'\in \mathcal{S}}p^{\MM}_h(s'|s,a) V_{h+1}^{\pi}(s'), \text{ and } Q_h^*(s,a) = r^{\MM}_h(s,a) + \sum_{s'\in \mathcal{S}}p^{\MM}_h(s'|s,a) V_{h+1}^{*}(s').$$

\textbf{Human-centric system.} The theme of the formulation and all our following results is a human-centric decision system where the machine acknowledges the suboptimal behavior of the human and makes advice on critical states to improve the reward. So the learning and optimization of our paper take the perspective of the machine (solving $\mathcal{M}^\mathtt{M}$) and do not seek to change the underlying human policy $\pi^\mathtt{H}.$

\section{The learning problem}
Now we discuss learning problems associated with the above human-machine adherence model. We consider two \textit{learning environments} for the problem:

$\mathcal{E}_1$ (Environment 1 -- partially known): the environment's state transition kernel $p$, the reward $r$, and the human's behavior policy $\pi^{\HH}$, are known; the human's adherence level $\theta$ is unknown.

$\mathcal{E}_2$ (Environment 2 -- fully unknown): the environment's state transition kernel $p$, the reward $r$, the human's behavior policy $\pi^{\HH}$, and the human's adherence level $\theta$ are unknown.

For $\mathcal{E}_1$, the goal is simply to learn the optimal policy under the unknown adherence level $\theta$. We develop a learning algorithm that outputs $\epsilon$-optimal advice policy and features better sample complexity compared to the vanilla application of problem-agnostic RL methods on $\mathcal{M}^{\MM}$. For $\mathcal{E}_2$, we know neither the environment nor the human's policy. Thus the learning problem entails learning the dynamics of both the environment and the human policy. We develop a provably convergent learning algorithm that outputs the optimal policy, and in addition, the learned advice policy only gives advice when necessary (choosing to defer for non-critical steps).

Our investigations on these two learning formulations highlight three points. First, the inherent structure of the human-machine interaction allows more sample-efficient algorithms (than the vanilla application of the off-the-shelf RL algorithms) both theoretically and empirically. Second, the knowledge of the underlying environment ($\mathcal{E}_1$ compared against $\mathcal{E}_2$) significantly, also unsurprisingly, reduces the sample complexity of the learning algorithm. Third, we establish a close connection between the formulation of the human-machine interaction with the problems of reward-free exploration \citep{jin2020reward} and constrained MDPs \citep{altman2021constrained}. 

\subsection{Main Results}\label{subsec_results}

We first state the technical results and then present the detailed algorithms and analyses in the subsequent section. 


\textbf{Theorem 1 (Environment $\mathcal{E}_1$, informal)} For environment $\mathcal{E}_1$, Algorithm \ref{alg:alg1} finds an $\epsilon$-optimal advice policy with a PAC sample complexity $O(H^2S^2A/\epsilon^2)$ with high probability.

Under environment $\mathcal{E}_1$, Theorem 1 gives a PAC sample complexity for the UCB-type (Upper-Confidence-Bound-type) Algorithm \ref{alg:alg1}. We remark that applying the existing problem-agnostic algorithms can only achieve a suboptimal order of sample complexity on the problem: $O(H^3S^2A/\epsilon^2)$ via the model-based algorithm \citep{dann2015sample} and $O(H^4SA/\epsilon^2)$ via the model-free algorithm \citep{jin2018q}\footnote{The authors obtain a regret bound instead of PAC sample complexity bound. However, they convert the regret bound to a PAC sample complexity bound in \cite[Section 3.1]{jin2018q}}. Specifically, the bound in \cite{dann2015sample} gives an additional factor of $H$ compared to the bounds in the original setting, where stationary transition density is assumed; this is due to the fact that though the adherence level $\theta$ is stationary, the transition becomes non-stationary when compounding $\theta$ and underlying transition of the human's underlying MDP. Also, we note that such an improvement on $H$ is not due to a reduction in the number of unknown parameters because the adherence level $\theta$ has a dimensionality of $SA.$ Indeed, the key to the improvement is the intrinsic structure of the human-machine problem enables a more sample-efficient design of the UCB algorithm (See Section \ref{sec:alg1_ana} for details). Moreover, we also provide another algorithm that finds an $\epsilon$-optimal advice policy with a sample complexity of $O(H^3SA/\epsilon^2)$ for $\mathcal{E}_2$ (See Algorithm \ref{alg:alg3} in appendix \ref{ap_pf_alg2} for details).

For environment $\mathcal{E}_2$, we assume no prior knowledge at all, and this makes the machine's problem no different than a generic RL problem. Thus we consider a slight twist of the machine's MDP with the notion of \textit{pertinent} advice. This twisted formulation enables richer analytical structures and draws interesting connections with several existing frameworks. Specifically, consider a new machine's MDP 
$\mathcal{M}^{\MM}_{\beta} \in \left(\mathcal{S}, \bar{\mathcal{A}}, H, p^{\MM}, r^{\MM}_{\beta}\right)$ which inherits everything from $\mathcal{M}^{\MM}\in \left(\mathcal{S}, \bar{\mathcal{A}}, H, p^{\MM}, r^{\MM}\right)$ except for the reward 
\begin{align}\label{eqn_penalty}
    r^{\MM}_{h,\beta}(s,a) = r_h^{\MM}(s, a) - \beta\cdot \mathbb{I}\{a \neq \text{defer}\},
\end{align}
where the $\mathbb{I}\{\cdot\}$ is the indicator function and $\beta>0$ is a constant. Under $\mathcal{M}^{\MM}_{\beta}$, we denote $V^{\pi}_{\beta}$ and $V^{*}_{\beta}$ the value functions of $\pi$ and the optimal value function, respectively, and the optimal policy $\pi_{\beta}^* \in \argmax_{\pi} V_{\beta}^{\pi}$. The new reward function enforces a penalization of $\beta$ for making advice and thus regularizes the number of machine advices throughout the horizon. In practice, providing advice to human at every step can be annoying in applications such as gaming, driving, or sports. Hence, it is crucial to prioritize and selectively deliver advice based on its criticalness -- which we term informally as \textit{pertinent} advice. For example, when the human is an expert and already achieves near-optimal performance, there is no need to give advice; also, when the human is under-performing, and the adherence level is low, there is also no need to give advice because it is unlikely to be taken.

\begin{proposition}\label{prop_gap}
    For all $s\in\mathcal{S}$ and $h\in[H]$ such that $\pi^*_{h,\beta}(s)\neq\text{defer}$, we have
    \begin{equation*}
    Q_h^*(s, \pi^*_{h,\beta}(s)) - V_h^{\pi^{\HH}}(s) \ge \beta.
    \end{equation*}
\end{proposition}

The proposition says that if the machine takes $\pi^*_{h,\beta}(s)$ and sticks with the optimal policy afterward, the reward will be at least $\beta$ more than that if the machine chooses to defer all the way till the end. In this light, we can rank the criticalness of making advice at different states by solving $\mathcal{M}^{\MM}_{\beta}$ with different $\beta$ which gives a better interpretation of this human-machine system.

\textbf{Theorem 2 (Environment $\mathcal{E}_2$, informal)} For $\mathcal{E}_2$, Algorithm \ref{alg:alg2} outputs a family of $\epsilon$-optimal policies $\{\hat{\pi}_{\beta}\}_{\beta > 0}$ for $\{\mathcal{M}_{\beta}^{\HH}\}_{\beta > 0}$ with $O(H^5SA/\epsilon^2)$ episodes such that the following inequality
\begin{align}\label{obj2}
V_{1,\beta}^*(s_1) - V_{1,\beta}^{\hat{\pi}_{\beta}}(s_1) \leq \epsilon 
\end{align}
holds uniformly for all $\beta>0$ with high probability.

Theorem 2 gives the sample complexity of Algorithm \ref{alg:alg2} which learns a near-optimal policy for all the models $\{\mathcal{M}_{\beta}^{\HH}\}_{\beta \geq 0}$ simultaneously. Such joint learning not only provides a family of policies for the human to customize $\beta$ according to her/his performance but also gives us a handle to understand which are the critical states where the human's policy can be significantly improved.

\section{Algorithms and Analyses}
In this section, we present the algorithms and analyses that achieve the results mentioned previously.

\subsection{UCB-based algorithm for $\mathcal{E}_1$}

\label{sec:alg1_ana}
Under $\mathcal{E}_1$, the machine works with a human with unknown adherence level $\theta.$ An important property of $\theta$ is as follows. Basically, it states that the team of human and machine achieves a higher optimal reward if the human has a higher adherence level. To emphasize the dependence on $\theta$, we write
$$V_{h}^{\pi}(s|\theta) = \mathbb{E}\left[\sum_{h'=h}^{H}r^{\MM}_{h'}(s_{h'}, a_{h'})\Big \vert s_{h} = s, \text{adherence parameter }\theta\right] \,\,\text{and}\,\,\, V_h^*(s|\theta) = \max_{\pi\in \Pi} V_h^{\pi}(s|\theta).$$
\begin{proposition}[Monotonicity property]\label{prop_monotone}
Suppose $\theta_1 \geq \theta_2$ holds entry-wise, then the following inequality holds for all $s\in\mathcal{S}$ and $h\in[H]$
$$V_h^*(s|\theta_1) \geq V_h^*(s|\theta_2).$$
\end{proposition}

Proposition \ref{prop_monotone} implies that finding an upper bound for the optimal value function reduces to finding an upper bound for $\theta$. Algorithm \ref{alg:alg1} follows this implication and maintains an optimistic estimate $\bar{\theta}^{t}$ for the true parameter $\theta$. For each episode, it generates the policy $\hat{\pi}_t$ pretending the $\bar{\theta}^{t}$ as true, and rolls out the episode according to $\hat{\pi}_t$. Then it updates the estimate with the new observations. The optimistic estimate $\bar{\theta}^{t}$ takes the form of a standard UCB form with a careful choice of the confidence width and we defer more details to Appendix \ref{ap_pf_alg1}. The algorithm shares the same intuition as other UCB-based algorithms that, with more and more observations, the confidence bound $\bar{\theta}^{t}$ will shrink to the true $\theta$, and so does the value functions.


\begin{algorithm}[ht!]
\caption{UCB-ADherence (UCB-AD)}
\label{alg:alg1}
\begin{algorithmic}[1]
\State Input: Target probability level $\delta$. 
\State Initialize $t = 1$, $\mathcal{D}_{t-1}=\emptyset$, and the optimistic estimate $\bar{\theta}^t = \bm{1}.$
\For{$t = 1, 2, \cdots $}
\State Solve the advice policy $\hat{\pi}^{t} = \argmax_{\pi} V^{\pi}(\cdot|\bar{\theta}^t)$ given the current optimistic estimate $\bar{\theta}^t$
\State Sample a new episode $z_t = \left\{s_1^t, a_1^{\MM,t}, a_1^{\HH,t}, r^t_1, \cdots, s_H^t, a_H^{\MM,t}, a_H^{\HH,t}, r_H^t\right\}$ following policy $\hat{\pi}^{t}$
\State Update $\mathcal{D}_t \leftarrow \mathcal{D}_{t-1} \cup \{z_{t}\}$ 
\State Update the optimistic estimate $\bar{\theta}^t\rightarrow \bar{\theta}^{t+1}$ based on $\mathcal{D}_t$ and $\delta$
\EndFor
\end{algorithmic}
\end{algorithm}

Theorem \ref{thm:UCB-AD} establishes an $(\epsilon, \delta)$-PAC result for Algorithm \ref{alg:alg1}.

\begin{theorem}
\label{thm:UCB-AD}
For any $\delta \in (0,1)$, $\epsilon \in (0, 1]$, and $T\in \mathbb{N}^+$, the number of policies among $\{\hat{\pi}^t\}_{t=1}^{T}$ from Algorithm \ref{alg:alg1} that are not $\epsilon$-optimal, i.e., $V_1^*(s_1)-V_1^{\hat{\pi}^{t}}(s_1)>\epsilon$, is bounded by $\tilde{O}\left(\frac{H^2S^2A}{\epsilon^2}\cdot \log\frac{1}{\delta}\right)$ with probability $1-\delta$.
\end{theorem}

The proof of the theorem mimics the analysis of \cite{dann2015sample}. One caveat in the analysis is that the original analysis of \cite{dann2015sample} focuses on a stationary setting where transition probabilities depend solely on state and action, remaining independent of the time horizon. However, even when the adherence level $\theta$ remains the same over time, the machine's MDP is non-stationary. An direct adoption is to enlarge the state space to incorporate the horizon step $h$, yet this will result in a sample complexity of $O(H^3S^2A/\epsilon^2)$, a worse dependency on $H$. The key is to reduce the upper bound analysis to the adherence level space and utilize Proposition \ref{prop_monotone} to convert that into a suboptimality gap with respect to the value function. This treatment gives the desirable bound in Theorem \ref{thm:UCB-AD} which also outperforms the bound from a direct application of results from \cite{azar2017minimax} to non-stationary MDPs.

\subsection{Reward-free exploration algorithm for $\mathcal{E}_2$}
$\mathcal{E}_2$ has more unknown parameters than $\mathcal{E}_1$ and thus it naturally entails more intense exploration. Moreover, the learning objective becomes more complex: we aim not only to learn the near-optimal policy but also to discern the pertinent advice. 

Algorithm \ref{alg:alg2} is based on the concept of \textit{reward-free exploration} (RFE)  \citep{jin2020reward}. Specifically, RFE algorithms  usually consist of an exploration phase and a planning phase. During the exploration phase, the algorithm collects trajectories from an MDP $\mathcal{M}$ without a pre-specified reward function. In the planning phase, it can compute near-optimal policies of $\mathcal{M}$, given any deterministic reward functions that are bounded. 


In our human-machine model, the machine observes $s_h \to a^{\MM} \to a^{\HH} \to s_{h+1}$, and the trajectory for episode $t$ is 
$z_t = \{s_1^t, a_1^{\MM,t}, a_1^{\HH,t}, r_1^t, s_2^t, a_2^{\MM,t}, a_2^{\HH,t}, r_2^t, \cdots, s_H^t, a_H^{\MM,t}, a_H^{\HH,t}, r_H^t\}$, where $a_h^{\MM,t} = \pi^t(s_h^t)$, $a_h^{\HH,t} \sim \mathbb{P}_h(\cdot|s_h^t, a_h^{\MM,t})$, and $s_{h+1}^t \sim \p_h(\cdot|s_h^t, a_h^{\HH,t})$. We denote $\hat{p}^{\MM,t}_h$ and $\hat{r}^{\MM,t}_h$ the empirical estimation for $p^{\MM}$ and $r^{\MM}_h$, and $n_h^t(s,a) = \sum_{i=1}^t\mathbb{I}{\left\{\left(s_h^i, a_h^{\MM, i}\right) = (s, a)\right\}}$ the number of times the machine gives advice $a$ at time $h$ and state $s$ in the first $t$ episodes. The key quantity in Algorithm \ref{alg:alg2} is
\begin{equation}\label{eqn_W}
    \begin{aligned}
        W_h^t(s,a) = \min\left(H, 16H^2\cfrac{\phi(n_h^t(s,a), \delta)}{n_h^t(s,a)}+\left(1+\frac{1}{H}\right)\sum_{s'}\hat{p}^{\MM,t}_h(s'|s,a)\max_{a'}W_{h+1}^t(s',a')\right),
    \end{aligned}
\end{equation}
where $W^t_{H+1}(s, a) = 0$ for $(s,a)\in \mathcal{S}\times \mathcal{A}$, and $\phi(n,\delta)$ grows at the order of $O(\log(n)+\log(1/\delta))$ and is specified in Theorem \ref{thm_RFE}.

Now we formally introduce our Algorithm \ref{alg:alg2}. The algorithm iteratively minimizes an upper bound defined by \eqref{eqn_W} which measures the uncertainty of a state-action pair, and the upper bound shrinks as the number of visits for the state-action pair increases. The algorithm stops when the upper bound is less than a pre-specified threshold. This algorithm is inspired by the RF-Express algorithm \citep{menard2021fast}, and there is a slight difference in the definition of $W_h^t(s,a)$, $\phi(n,\delta)$ and the stopping rule. In our application, the reward $r^{\MM}$ is stochastic and we need to take care of the estimation error; while in \cite{menard2021fast}, the algorithm does not need to deal with the reward at all.
\begin{algorithm}[ht!]
\caption{: RFE-$\beta$}
\label{alg:alg2}
\begin{algorithmic}[1]
\State Input: $\epsilon, \delta$, \text{ and }user-specified $\{\beta_i\}_{i \in \mathcal{I}}$, where $\mathcal{I}$ could be any set where $\beta_i \in [0,H)$
\State \textbf{Stage 1: Reward-free exploration}
\State Initialize $t = 1$ and $W^t_h(s,a) = H$ for all $(s,a) \in \mathcal{S}\times\mathcal{A}$ 
\State Compute ${\pi}^t$ so that ${\pi}_h^t(s) = \argmax_{a\in\mathcal{A}}W_h^t(s, a)$ (see \eqref{eqn_W})
\While{$W_1^t(s_1, \pi^t(s_1))+4e\sqrt{W_1^t(s_1, \pi^t(s_1))} > \epsilon/H$}
\State Sample trajectory $z_t = \{s_1^t, a_1^{\MM,t}, a_1^{\HH,t},r_1^t, \cdots, s_H^t, a_H^{\MM,t}, a_H^{\HH,t}, r_H^t\}$ following $\pi^t$
\State update $t \leftarrow t+1$, $\mathcal{D} \leftarrow \mathcal{D}\cup \{z_t\}$, $\hat{p}^{\MM,t}_h(s'|s,a)$, $\hat{r}^{\MM,t}_h(s,a)$, and $W_h^t(s,a)$
\EndWhile
\State \textbf{Stage 2: Policy identification}
\State Use planning algorithms to output optimal advice policy $\{\hat{\pi}_{\beta_i}^{\tau}\}_{i\in\mathcal{I}}$ for $\left\{\left(\mathcal{S},\bar{\mathcal{A}},H,\hat{p}^{\MM}, \hat{r}^{\MM}_{\beta_i}\right)\right\}_{i\in\mathcal{I}}$
\end{algorithmic}
\end{algorithm}

\begin{theorem}\label{thm_RFE}
For $\delta \in (0,1)$, $\epsilon \in (0, 1]$, and $\phi(n,\delta) = 6\log(4HSA/(\epsilon\delta))+S\log(8e(n+1))$, with probability $1-\delta$, Stage 1 of Algorithm \ref{alg:alg2} stops in $\tau$ episodes and
$$\tau \leq C_1\cfrac{H^5SA}{\epsilon^2}\left(6\log(4HSA/(\epsilon\delta))+S\right),$$
where $C_1 = \tilde{O}(\log(HSA))$. Moreover, $\{\hat{\pi}^{\tau}_{\beta}\}_{\beta > 0}$ have the following property
\begin{align*}
    P\left(V_{1,\beta}^*(s_1) - V_{1,\beta}^{\hat{\pi}^{\tau}_{\beta}}(s_1) \leq \epsilon \,\,\text{uniformly for all $\beta \in [0, H)$}\right) > 1-\delta.
\end{align*}
\end{theorem}

Theorem \ref{thm_RFE} ensures that Algorithm \ref{alg:alg2} provides sample estimation for the underlying MDP such that all the policy $\{\hat{\pi}_{\beta}^{\tau}\}_{\beta \in [0,H)}$ for pertinent advice are near optimal. The proof is a direct application of the RF-Express \citep{menard2021fast}, except that we have to take care of the estimation error in $\hat{r}^{\MM}$. Although Algorithm \ref{alg:alg2} has the uniform convergence property for any number of bounded reward functions, it can also be used the same way as Algorithm \ref{alg:alg1}, to find the $\epsilon$-optimal policy for $\mathcal{M}^{\MM}$ if provided with the non-penalized reward function $\hat{r}^{\MM}$. In this context, we can modify RFE-$\beta$ so that with high probability, it solves $\mathcal{M}^{\MM}$ with a sample complexity of $O(H^3SA/\epsilon^2)$ (See Algorithm \ref{alg:alg3} in Appendix \ref{ap_pf_alg2} for details).

\textbf{CMDP for pertinent advice.} The algorithm RFE-$\beta$ solves a class of problems $\{\mathcal{M}_\beta^{\mathtt{M}}\}_{\beta > 0}$ simultaneously for all the $\beta$'s and it measures the pertinence of advice by $\beta$. However, sometimes humans lack a quantitative view of how large a $\beta$ value should be considered as pertinent. Here, we introduce a different perspective on how the human should rank the importance of advice, framing it as ``in $H$ steps, I want advice no more than $D$ times'', and formulate this as a CMDP problem
\begin{equation}\label{eqn_cmdp}
    \begin{aligned}
        \max_{\pi} \,\,\,& \mathbb{E}^{\pi}\left[\sum_{h=1}^H r^{\MM}(s_h, a_h)\right]\,\,\,\,\,\,\,
        s.t.\,\,\, \mathbb{E}^{\pi}\left[\sum_{h=1}^H \mathbb{I}\{a_h \neq \text{defer}\}\right] \leq D,
    \end{aligned}
\end{equation}
where $D \in (0, H)$. From the standard primal-dual theorem, this formulation is closely related to the penalty $\beta$ in \eqref{eqn_penalty}, for the reason that we can treat $\beta$ as a dual variable for the constraint $D$. We refer the reader to the proof of Corollary \ref{cor_CMDP} in Appendix \ref{ap_pf_alg2} for details.

Now we present the CMDP method for pertinent advice. After stage 1 of RFE-$\beta$, we solve
\begin{equation}\label{eqn_sample_cmdp}
    \begin{aligned}
        \max_{\pi} \,\,\,\hat{\mathbb{E}}^{\pi}\left[\sum_{h=1}^H \hat{r}^{\MM,\tau}(s_h, a_h)\right]\hspace{5mm}
        s.t.\,\,\, \hat{\mathbb{E}}^{\pi}\left[\sum_{h=1}^H \mathbb{I}\{a_h \neq \text{defer}\}\right] \leq D,
    \end{aligned}
\end{equation}
where $\hat{\mathbb{E}}$ is the expectation with the underlying transition being $\hat{p}^{\MM,\tau}$. The next corollary states that $\hat{\pi}^{\tau}_D$, the solution for \eqref{eqn_sample_cmdp}, is a near-optimal policy for the CMDP \eqref{eqn_cmdp}.

\begin{corollary}\label{cor_CMDP}
In the same setting of Theorem \ref{thm_RFE}, for $\delta \in (0,1)$ and $\epsilon \in (0, 1]$, with probability $1-\delta$, for all $D \in (0,H)$, $\hat{\pi}^{\tau}_D$ is a near-optimal solution for the original CMDP \eqref{eqn_cmdp} such that 
\begin{equation}
    \begin{aligned}
        V_1^{\hat{\pi}^{\tau}_D}(s_1) \geq  V_1^{\pi^*_D}(s_1) - 2\epsilon,\hspace{0.5cm} \text{and }\hspace{0.5cm}\mathbb{E}^{\hat{\pi}^{\tau}_D} \left[\sum_{h=1}^H \mathbb{I}\{a_h \neq \text{defer}\}\right] \leq D+\epsilon
    \end{aligned}
\end{equation}
where ${\pi^*_D}$ is the optimal solution for \eqref{eqn_cmdp}.
\end{corollary}

Corollary 1 also implies that RFE-$\beta$ can compute near-optimal policies of CMDP \eqref{eqn_cmdp} for \textbf{all} the constraints $D \in [0,H)$, with a sample complexity of $O(H^5SA/\epsilon^2)$. Compared to other CMDP learning algorithms (for example, $O(H^2S^3A/\epsilon^2)$ in \cite{kalagarla2021sample}), the sample complexity of Corollary 1 features a lower order in $S$. Moreover, the near-optimal result holds for all constraints $D \in [0,H)$, and for other CMDP learning algorithms, the result only holds for a pre-specified $D$.


\section{Numerical Experiment}\label{sec:exp}
We perform numerical experiments under two environments: \textit{Flappy Bird} \citep{williams2023computational} and \textit{Car Driving} \cite{meresht2020learning}. Both Atari game-like environments are suitable and convenient for modeling human behavior while retaining the learning structure for the machine. We focus on the flappy bird environment here and defer the car driving environment to Appendix \ref{apnd:B}.

\textbf{Flappy Bird Environment.} We consider a game map of a 7-by-20 grid of cells. Each cell can be empty, contain a star, or act as a wall. The goal is to navigate the bird across the map from left to right and collect as many stars as possible. However, colliding with a wall or reaching the (upper and lower) boundaries leads to the end of the game. An example map is displayed in Figure \ref{fig_flappy}, which splits into three phases: the first phase contains almost only stars and no walls, the second phase contains almost only walls and very few stars, and the third phase contains both stars and walls.
\begin{figure}[ht!]
    \centering
\includegraphics[scale=0.25]{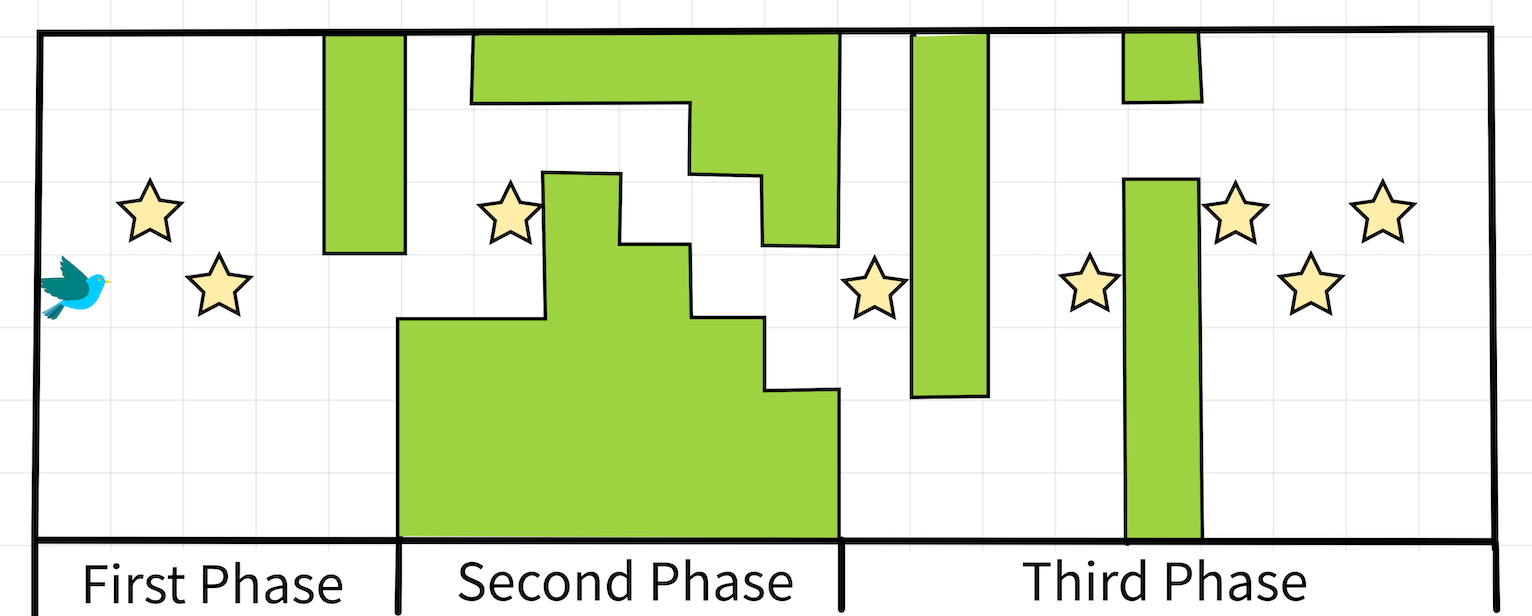}
\vspace{-0.2cm}
    \caption{Flappy Bird environment: player needs to navigate the bird to avoid walls and collect stars.}
    \label{fig_flappy}
\end{figure}

We define the state space as the current locations of the bird on the grid, represented by coordinates $(x, y) \in \mathbb{Z}^2$, with a total of $7 \times 20 = 140$ states.  Regarding the action space, we define it as $\mathcal{A} = \{\text{Up, Up-Up, Down}\}$. Each action causes the bird to move forward by one cell. In addition, the ``Up'' action moves the bird one cell upwards, the ``Up-Up'' action moves it two cells upwards, and the ``Down'' action moves it one cell downwards. The MDP has a reward as a function of state only. We will get a reward of $1$ when the current state (location) has a star and otherwise $0$. To model human behavior, we consider two sub-optimal human policies:  \textbf{Policy Greedy}, which prioritizes collecting stars in the next column, and \textbf{Policy Safe}, which focuses on avoiding walls in the next column.  If there is no preferred action available, both policies maintain a horizontal zig-zag line by alternating between ``Up'' and ``Down''. For adherence level $\theta$, we assume for all $s\in \mathcal{S}$ and $h=1,...,H$, the human will adhere to the advice with probability $0.9$ except the aggressive advice ``Up-up'' (which moves too fast vertically) with adherence level $0.7$. We compare the following algorithms:
\begin{itemize}
    \item UCB-ADherence (UCB-AD): Algorithm \ref{alg:alg1} that finds the $\epsilon$-optimal advice policy.
    \item RFE-ADvice (RFE-AD): Algorithm \ref{alg:alg3}, a variant of RFE-$\beta$ that finds the $\epsilon$-optimal policy. 
    \item RFE-$\beta$: Algorithm \ref{alg:alg2} that outputs pertinent advice policy by exploring then planning. 
    \item RFE-CMDP: A variant of RFE-$\beta$ that solves the CMDP \eqref{eqn_sample_cmdp} after exploring.
\end{itemize}

Figure \ref{fig:regret_bird_greedy} and \ref{fig:regret_bird_safe} present the results for the two algorithms UCB-AD and RFE-AD for the environment $\mathcal{E}_1$. It also includes the state-of-the-art algorithm EULER \citep{zanette2019tighter} that achieves a generic minimax optimal regret. From the regret plot, UCB-AD outperforms both RFE-AD and EULER. This advantage is attributed to UCB-AD's effective utilization of the information and structure of the underlying MDP. These results also show that our tailored algorithms UCB-AD and RFE-AD are much more efficient
than directly applying problem-agnostic RL algorithms in the adherence model. We further test UCB-AD with different $\theta$'s: with $\theta_1$, $\theta(a,s)\equiv 0.8$ and with $\theta_2$, $\theta(a,s)\equiv 0.4$. Figure \ref{fig:regret_UCB_AD_thetas} shows the relationship between the regret of UCB-AD and $\theta$: for both policies, UCB-AD can achieve smaller regret with higher $\theta$. Intuitively, a high adherence level implies a high probability of following the advice instead of taking $\pi^{\HH}$, which will reduce the regret caused by the suboptimality of $\pi^{\HH}$.  
\begin{figure}[ht!]
\centering
\begin{subfigure}[c]{0.32\textwidth}
\includegraphics[height=0.17\textheight]{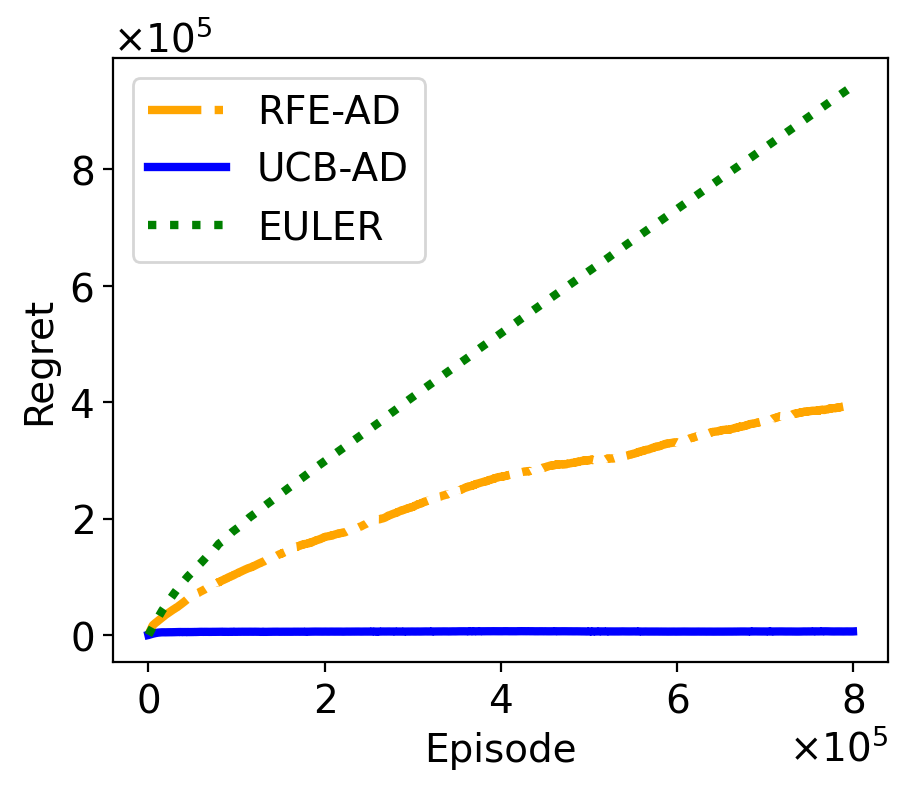}
\vspace{-0.2cm}
   \caption{Regrets of Policy Greedy.}
   \label{fig:regret_bird_greedy} 
\end{subfigure}
\hfill
\begin{subfigure}[c]{0.32\textwidth}
\includegraphics[height=0.17\textheight]{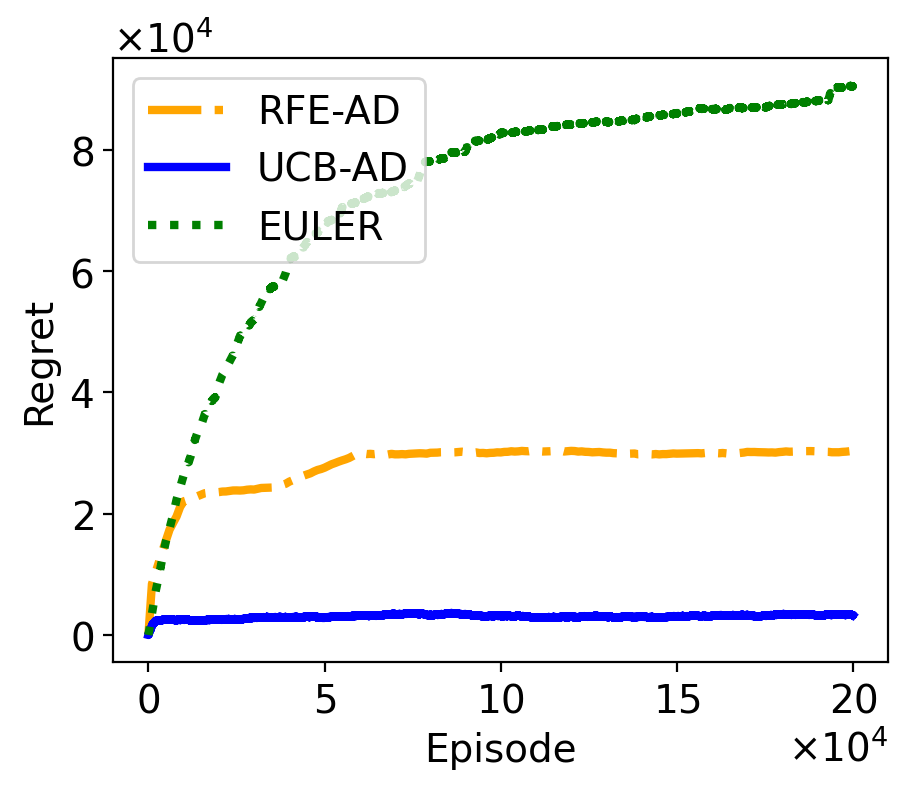}
\vspace{-0.2cm}
   \caption{Regrets of Policy Safe.}
   \label{fig:regret_bird_safe} 
\end{subfigure}
\hfill
\begin{subfigure}[c]{0.32\textwidth}
\includegraphics[height=0.17\textheight]{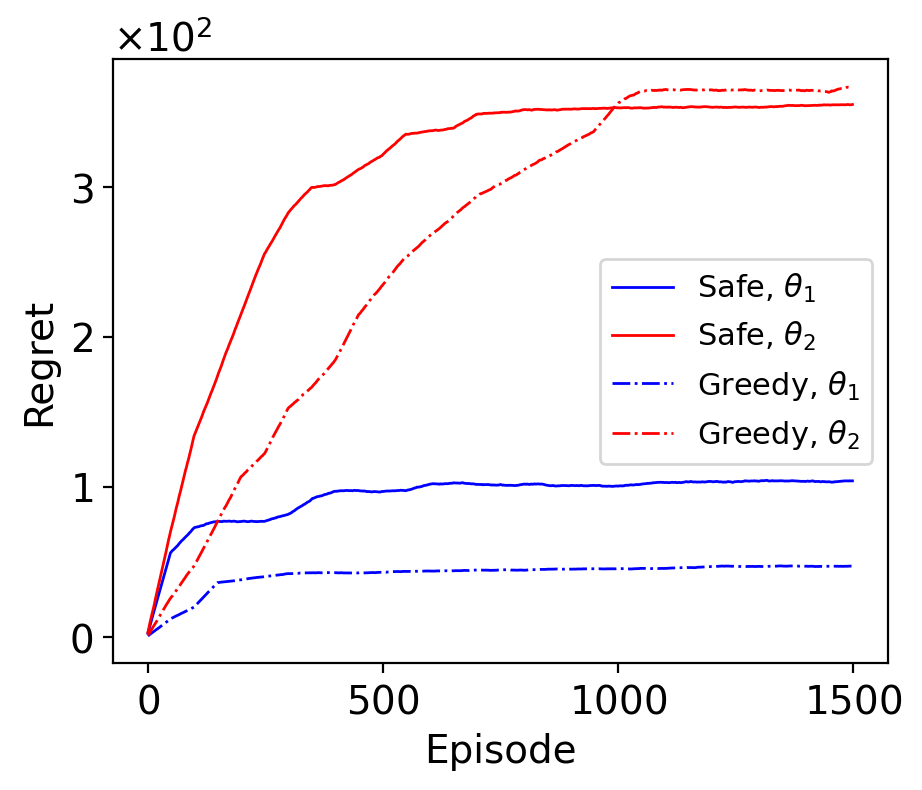}
\vspace{-0.2cm}
   \caption{Regrets of UCB-AD.}
   \label{fig:regret_UCB_AD_thetas} 
\end{subfigure}
\vspace{-0.2cm}
\caption{The regrets for learning the optimal advice for  Policy Greedy and Policy Safe. Figure \ref{fig:regret_bird_greedy}, \ref{fig:regret_bird_safe} show the regrets of RFE-AD, UCB-AD, and EULER for two policies respectively. Figure \ref{fig:regret_UCB_AD_thetas} shows the regrets of UCB-AD for two policies under different $\theta$'s. }
\label{fig:regret_bird}
\end{figure}

Figure \ref{fig:interp_figs} summarizes results for three policies under the environment $\mathcal{E}_2$, namely RFE-$\beta$, RFE-CMDP, and UC-CFH, a provably convergent CMDP algorithm \citep{kalagarla2021sample}, under Policy Safe. In Figure \ref{fig:cvg_beta}, we see that RFE-$\beta$ exhibits convergence for different $\beta$'s, and this empirically corroborates the theoretical finding. In Figure \ref{fig:cvg_CMDP_FR}, we compare RFE-CMDP and UC-CFH  under a simpler environment with the advice budget being 1 ($D = 1$). We observe that RFE-CMDP shows a marginal performance advantage over UC-CFH in terms of the convergence rate. More importantly, Figure \ref{fig:value_post_CMDP} shows by only using the estimated transition kernel after learning for $D = 1$ (Figure \ref{fig:cvg_CMDP_FR}), RFE-CMDP is able to obtain near-optimal policy for problem instances with different advice budgets ($D = 2,3,4$ and $5$). However, UC-CFH fails to explore the whole transition kernel sufficiently and can only output the near-optimal policy for the original problem instance. Moreover, RFE-CMDP is more sample efficient with respect to the advice budget, because for UC-CFH, we have to run multiple times with different advice budget parameters to get a near-optimal policy for all of them.

\begin{figure}[ht!]
\centering
\begin{subfigure}[c]{0.32\textwidth}
\includegraphics[height=0.16\textheight]{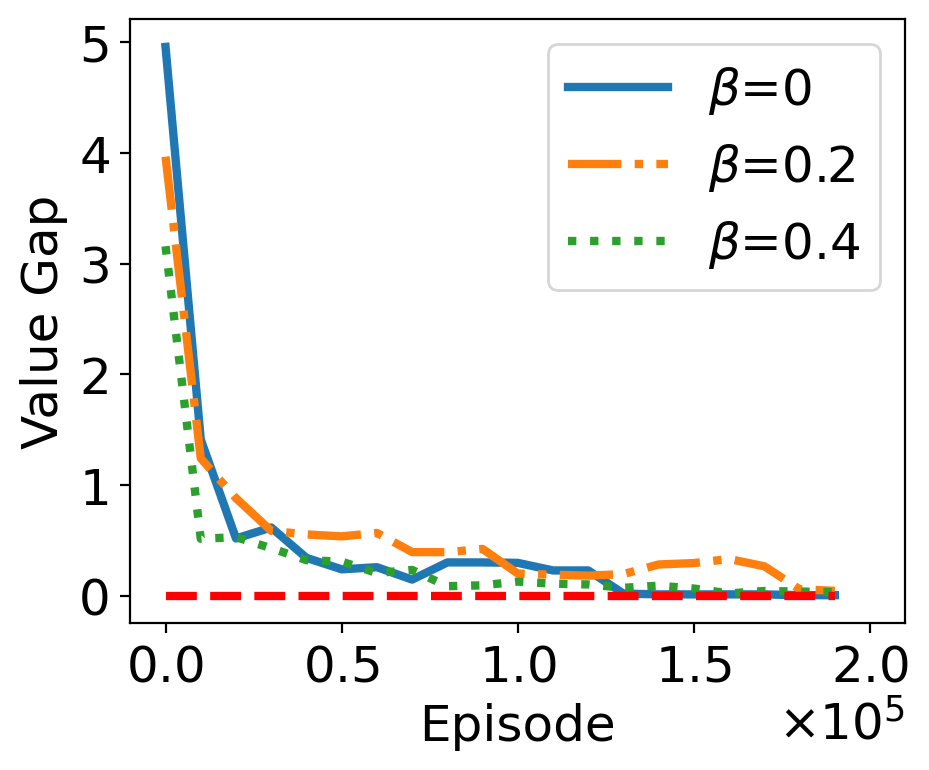}
\vspace{-0.3cm}
   \caption{Value gaps of RFE-$\beta$.}
   \label{fig:cvg_beta} 
\end{subfigure}
\hfill
\begin{subfigure}[c]{0.32\textwidth}
\includegraphics[height=0.16\textheight]{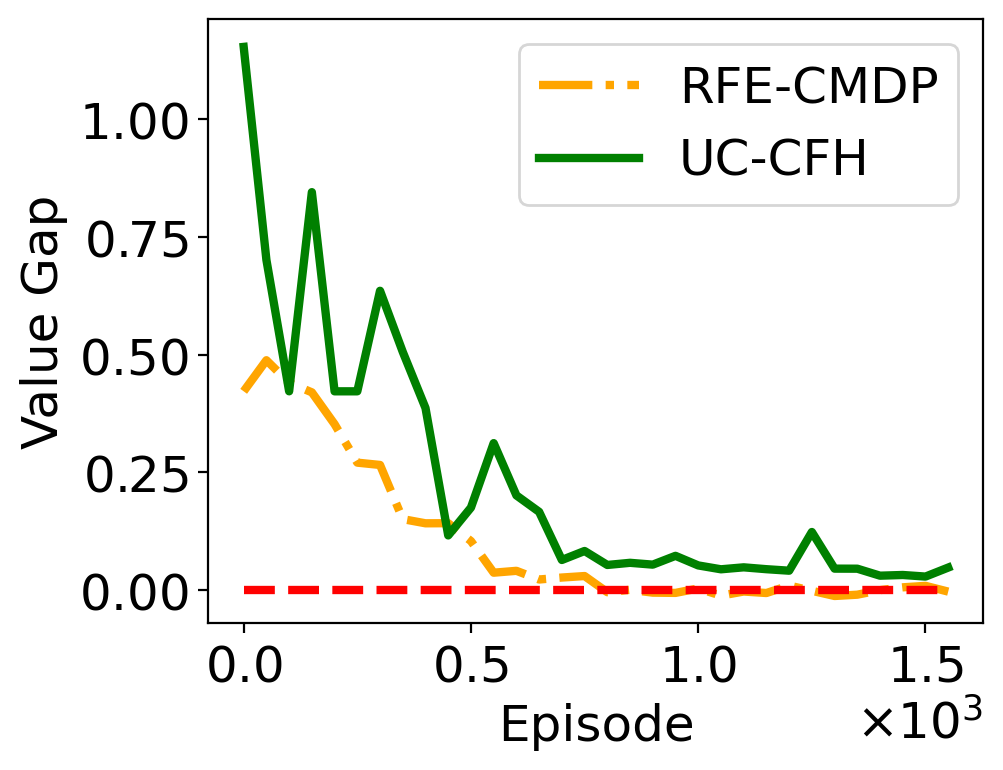}
\vspace{-0.3cm}
   \caption{Value gaps w.r.t. episode.}
   \label{fig:cvg_CMDP_FR} 
\end{subfigure}
\hfill
\begin{subfigure}[c]{0.32\textwidth}
\includegraphics[height=0.16\textheight]{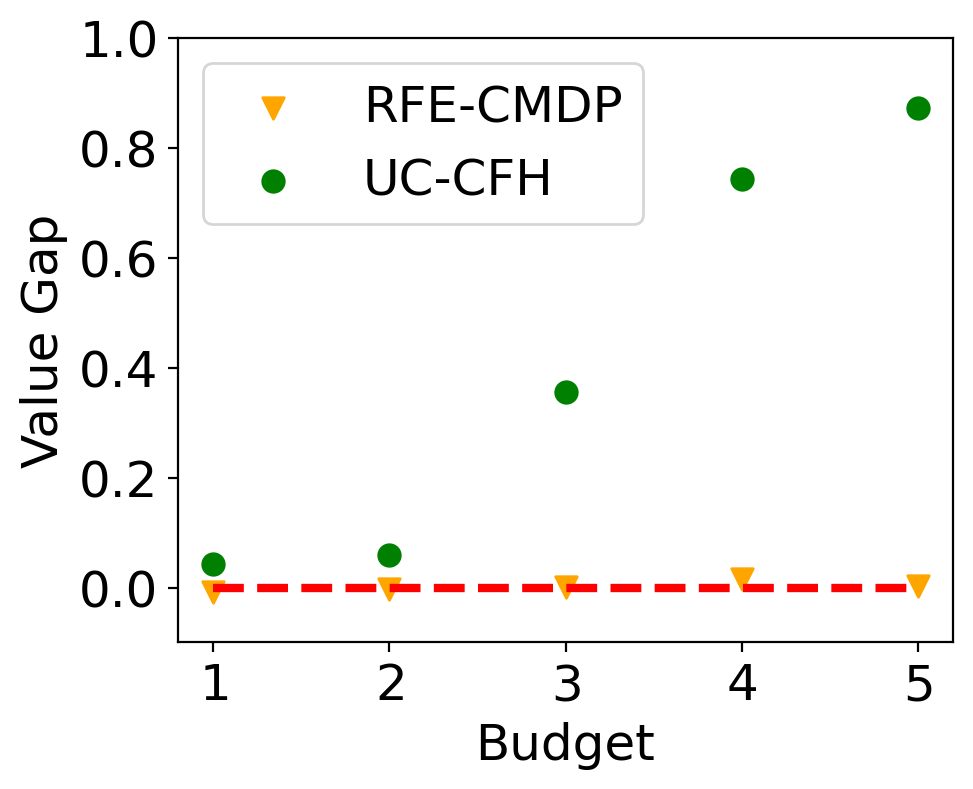}
\vspace{-0.3cm}
   \caption{Value gaps when evaluating.}
   \label{fig:value_post_CMDP} 
\end{subfigure}
\vspace{-0.2cm}
\caption{The performances of making pertinent advice. The value gap is defined as the difference between the value of current policy and the optimal values, with the red dashed
line as the benchmark for $0$ loss of the policy. Figure \ref{fig:cvg_beta} shows the convergence of RFE-$\beta$ under difference $\beta$'s. Figure \ref{fig:cvg_CMDP_FR} compares the convergences of RFE-CMDP and UC-CFH. Figure \ref{fig:value_post_CMDP} evaluates performance of policy learned from learning episodes in Figure \ref{fig:cvg_CMDP_FR}.}
\label{fig:interp_figs}
\end{figure}
\begin{figure}[ht!]
\centering
\begin{subfigure}{0.45\textwidth}
   \includegraphics[width=1\linewidth]{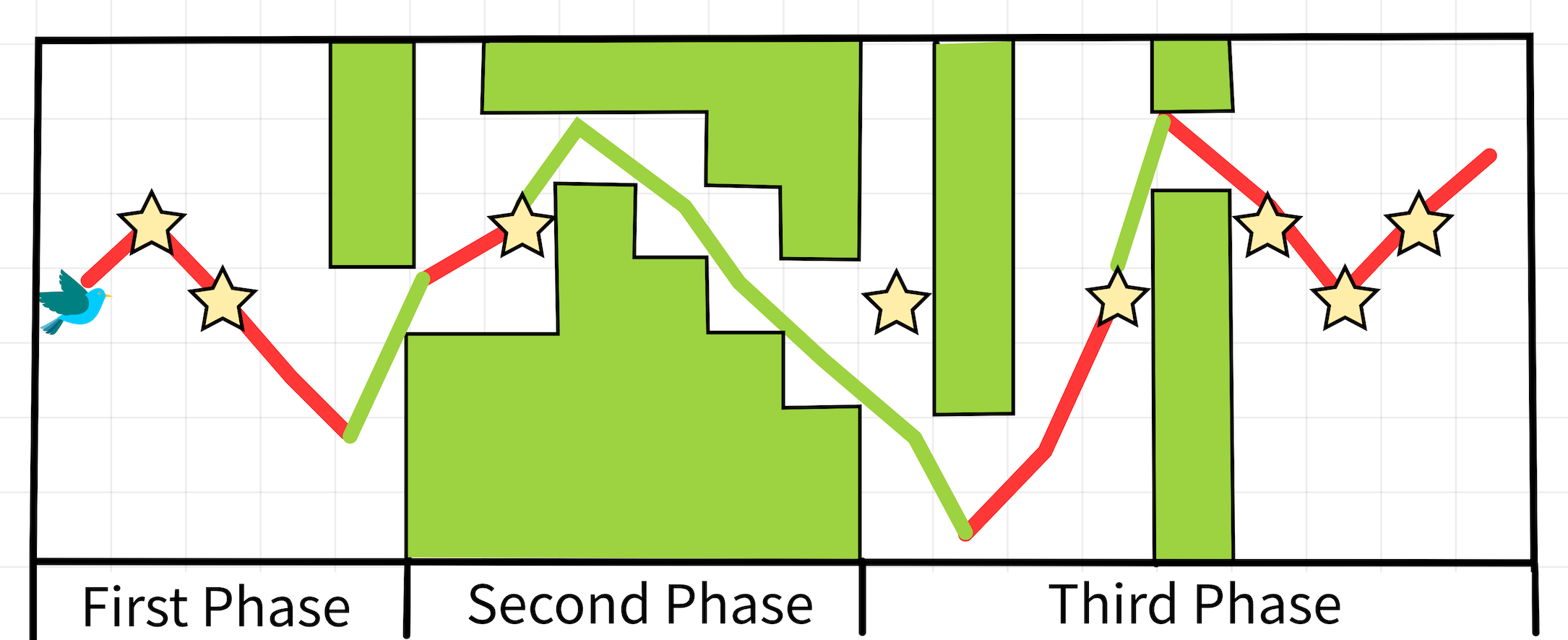}
   \vspace{-0.2cm}
   \caption{Trajectory of Policy Greedy.}
   \label{fig:ad_to_greedy} 
\end{subfigure}
\hfill
\begin{subfigure}{0.45\textwidth}
   \includegraphics[width=1\linewidth]{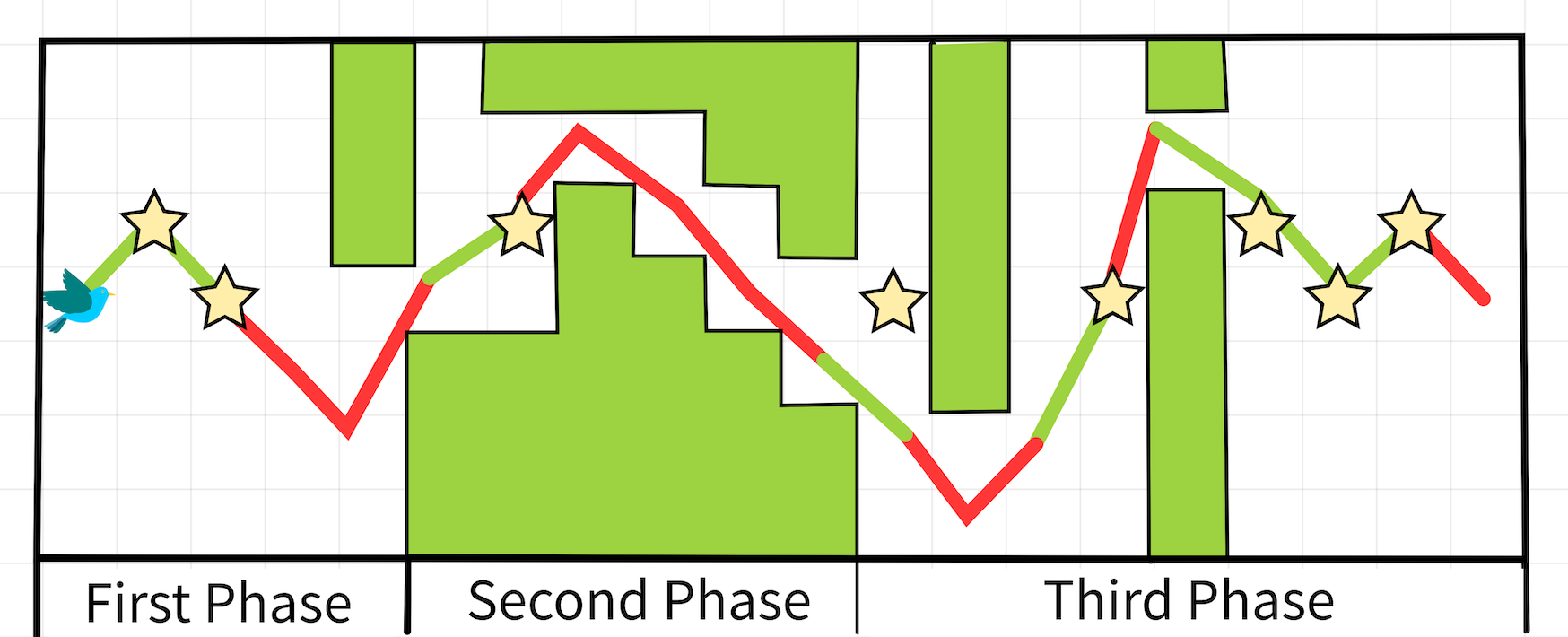}
   \vspace{-0.2cm}
   \caption{Trajectory of Policy Safe.}
   \label{fig:ad_to_safe}
\end{subfigure}
\vspace{-0.2cm}
\caption[Advice Trajectories.]{Typical trajectories of two policies' types. The red color means the machine defers and the green color means the machine advises.}
\label{fig:advice_traj_bird}
\end{figure}
Lastly, we show that RFE-$\beta$ is capable of generating pertinent advice for different policies. Figure  \ref{fig:advice_traj_bird} displays representative trajectories of two policies playing the game while receiving guidance from the machine, which follows $\hat{\pi}_{\beta}$ trained in the experiment of Figure \ref{fig:regret_bird}. By setting $\beta = 0.3$, the machine outputs a policy that only gives advice when necessary: Since Policy Greedy behaves well in the first phase, the machine almost only gives advice in the second phase and the third phase; Similarly, the machine almost only gives advice in the first phase and the third phase, and choose to defer most of the time when Policy Safe is in the second phase.

\newpage

\bibliographystyle{plainnat} 
\bibliography{sample.bib}

\newpage

\appendix
\section{Proofs}
In this appendix, we present the proofs for Section \ref{sec_model} in \ref{ap_sec2}, the proofs and other details for Algorithm \ref{alg:alg1} in Section \ref{ap_pf_alg1}, and the proofs and other details for Algorithm \ref{alg:alg2} in Section \ref{ap_pf_alg2}.

\subsection{Proofs for Section \ref{sec_model}}\label{ap_sec2}
\begin{proof}[Proof of Proposition \ref{prop_monotone}]
    To show $V^*(s|\theta_1) \geq V^*(s|\theta_2)$, it suffices to prove that there exists $\pi\in\Pi$ such that $V^{\pi}(s|\theta_1) \geq V^*(s|\theta_2)$. Denote $\pi^{*,\theta_2}$ the optimal deterministic policy such that $\pi^{*,\theta_2} = \argmax_{\pi\in\Pi}V^{\pi}(s|\theta_2)$. Based on $\pi^{*,\theta_2}$, we are going to construct a randomized policy $\tilde{\pi}^{\theta_1}$ such that $V^{\tilde{\pi}^{\theta_1}}(s|\theta_1) = V^*(s|\theta_2)$, and in the end, we will conclude the proof by showing $V^*(s|\theta_1)\geq V^{\tilde{\pi}^{\theta_1}}(s|\theta_1)$.

For any $h\in[H]$ and $s\in\mathcal{S}$, with a slight abuse of notation denote $a = \pi_h^{*,\theta_2}(s)$, the deterministic action of policy $\pi_h^{*,\theta_2}$. If $a = \text{defer}$, the randomized advice policy $\tilde{\pi}^{\theta_1}$ is defined as $\tilde{\pi}_h^{\theta_1}(s) = \text{defer}$; if $a \neq \text{defer}$, we have
\begin{equation*}
    \begin{aligned}
        \tilde{\pi}_h^{\theta_1}(s) = \begin{cases}
a, &\text{with probability $\frac{\theta_{2}(s,a)-\pi^{\HH}(a|s)}{\theta_{1}(s,a)-\pi^{\HH}(a|s)}$}\\
\text{defer},\,\,\,\,\,\,\, &\text{with probability $1-\frac{\theta_{2}(s,a)-\pi^{\HH}(a|s)}{\theta_{1}(s,a)-\pi^{\HH}(a|s)}$}\\
\end{cases}
    \end{aligned}
\end{equation*}
where we have assumed that $\theta_i(s,a) \geq \pi_h^{\HH}(a|s)$, i.e, the advice has no negative effect on the human's probability of taking action $a$ at state $s$ for $i=1,2$. Next, we have to verify $V^{\tilde{\pi}^{\theta_1}}(s|\theta_1) = V^*(s|\theta_2)$ by showing that for any $(a,s,h) \in (\bar{\mathcal{A}},\mathcal{S},[H])$,
\begin{align}\label{eqn_p1_dynamics}
    P^{\tilde{\pi}^{\theta_1}}_h(a^{\HH} = a|s,\theta_1) = P_h^{\pi_h^{*,\theta_2}}(a^{\HH} = a|s,\theta_2).
\end{align}
To see this, for $a = \pi_h^{*,\theta_2}(s)$ and $a \neq \text{defer}$, we have
\begin{equation*}
    \begin{aligned} P^{\tilde{\pi}^{\theta_1}}(a^{\HH}= a|s,\theta_1) &= \frac{\theta_2(s,a)-\pi_h^{\HH}(a|s)}{\theta_1(s,a)-\pi_h^{\HH}(a|s)}\theta_1(s,a) + \left(1-\frac{\theta_2(s,a)-\pi_h^{\HH}(a|s)}{\theta_1(s,a)-\pi_h^{\HH}(a|s)}\right)\pi_h^{\HH}(a|s)\\
    &= \theta_2(s,a)\\
    &= P_h^{\pi_h^{*,\theta_2}}(a^{\HH} = a|s,\theta_2).
    \end{aligned}
\end{equation*}
For $a' \neq \pi_h^{*,\theta_2}(s)$ and $a \neq \text{defer}$, we have
\begin{equation*}
    \begin{aligned} P^{\tilde{\pi}^{\theta_1}}(a^{\HH}= a'|s,\theta_1) &= \frac{\theta_2(s,a)-\pi_h^{\HH}(a|s)}{\theta_1(s,a)-\pi_h^{\HH}(a|s)}\cdot(1-\theta_1(s,a))\cdot \frac{\pi_h^{\HH}(a'|s)}{1-\pi_h^{\HH}(a|s)}\\
    &\hspace{4mm}+ \left(1-\frac{\theta_2(s,a)-\pi_h^{\HH}(a|s)}{\theta_1(s,a)-\pi_h^{\HH}(a|s)}\right)\pi_h^{\HH}(a'|s)\\
    &=(1-\theta_2(s,a))\frac{\pi_h^{\HH}(a'|s)}{1-\pi_h^{\HH}(a|s)}\\
    &= P_h^{\pi_h^{*,\theta_2}}(a^{\HH} = a'|s,\theta_2).
    \end{aligned}
\end{equation*}
For the last case, it is obvious that if $\pi_h^{*,\theta_2}(s) = \text{defer}$, the dynamics of choosing defer are independent of $\theta$, and this concludes proving \eqref{eqn_p1_dynamics}.

By showing \eqref{eqn_p1_dynamics}, we know that $V^{\tilde{\pi}^{\theta_1}}(s|\theta_1) = V^{\pi_h^{*,\theta_2}}(s|\theta_2) = V^*(s|\theta_2)$. Because we are working on a finite-horizon discrete MDP, from Bellman's equation, we know that the optimal value functions for the class of deterministic policies will be the same as those for the class of random policies. Therefore, we have $V^{*}(s|\theta_1) \geq V^{\tilde{\pi}^{\theta_1}}(s|\theta_1)$, and this concludes the proof.
\end{proof}

\begin{proof}[Proof of Proposition \ref{prop_gap}]
    
From the Bellman's equation, we know that if $a = \argmax_{a\in\bar{\mathcal{A}}}Q^*_{h,\beta}(s,a)$ and $a \neq \text{defer}$, we have $r^{\MM}_h(s,a) - \beta + \sum_{s'\in\mathcal{S}}p_h^{\MM}(s'|s,a)V_{h+1}^*(s') \geq V_{h,\beta}^{\pi^{\HH}}(s)$, and therefore
\begin{equation*}
    \begin{aligned}
        \beta &\leq  r^{\MM}_h(s,a) + \sum_{s'\in\mathcal{S}}p_h^{\MM}(s'|s,a)V_{h+1,\beta}^*(s') - V_{h,\beta}^{\pi^{\HH}}(s).
    \end{aligned}
\end{equation*}
By observing that $V_{\beta}^{\pi^{\HH}} = V^{\pi^{\HH}}$ (always deferring has no penalty), and $V_{\beta}^* \leq V^*$, we have 
\begin{equation*}
    \begin{aligned}
        \beta &\leq r^{\MM}_h(s,a) + \sum_{s'\in\mathcal{S}}p_h^{\MM}(s'|s,a)V_{h+1}^*(s') - V_h^{\pi^{\HH}}(s)\\
        &= Q_h^*(s,a) - V_h^{\pi^{\HH}}(s).\\
    \end{aligned}
\end{equation*}
\end{proof}

\subsection{Supplementary Materials for Algorithm \ref{alg:alg1}}
\label{ap_pf_alg1}
In this section, we first formally state the parameter updating rules in Algorithm \ref{alg:alg1} and define the related functions and parameters within the algorithm in Section \ref{sec:A.2.alg}. Subsequently, in Section \ref{sec:A.2.aa}, we present several useful lemmas to sketch the proof for Theorem \ref{thm:UCB-AD}. Then, in Section \ref{sec:A.2.pf}, we prove all statements under $\mathcal{E}_1$.

\subsubsection{Notations of Algorithm \ref{alg:alg1}}
\label{sec:A.2.alg}
In this section, we first state the parameter updating rules in Algorithm \ref{alg:alg1} under $\mathcal{E}_1$ as follows:
\begin{align*}
    \hat{\theta}^t(s,a)=\frac{1}{n^t(s,a)}\sum\limits_{i=1}^{t}\sum\limits_{h=1}^{H}\mathbb{I}(s_h^i=s,a_h^{\MM,i}=a,a^{\HH,i}_h=a),    
\end{align*}
where $n^t(s,a)=\sum\limits_{i=1}^{t}\sum\limits_{h=1}^{H}\mathbb{I}(s_h^i=s,a_h^{\MM,i}=a)$ for all $s\in\mathcal{S}$, $a\in\mathcal{A}$ and $t\leq T.$
\begin{align*}
    \bar{\theta}^t(s,a)
    &=\min\left\{1,
    \hat{\theta}^t(s,a)+\frac{C\left(\hat{\theta}^t(s,a),n^t(s,a),T,\delta\right)}{\sqrt{n^t(s,a)}}
    \right\},
\end{align*}
where
\begin{align*}
    C\left(\theta,n,T,\delta\right)=
    \min&\left\{
        2\sqrt{\log(12SAT/\delta)},
        \sqrt{2\theta(1-\theta)\log\frac{12SAT}{\delta}}+\frac{7\sqrt{n}}{3n-1}\cdot\log\frac{12SAT}{\delta},\right.\\
        &\left.\left(
            \frac{1+\sqrt{1+4\left(\max\left\{0, \sqrt{\theta(1-\theta)}-\sqrt{\frac{2\log(SAT/\delta)}{n-1}}\right\}\right)^2}}{2}-\theta
        \right)\cdot\sqrt{n}
    \right\}.
\end{align*}

Here, by definition $\bar{\theta}^t(s,a)$ becomes the largest element in the following set
\begin{align}
    \label{set:alg2_theta}
    \left\{
    \theta\in[0,1]: 
        \right.
    &|\theta- \hat{\theta}^t(s,a)|
    \leq2\sqrt{\frac{\log(12SAT/\delta)}{n^t(s,a)}},\\
    &|\theta- \hat{\theta}^t(s,a)|\leq\sqrt{\frac{2\hat{\theta}^t(s,a)(1-\hat{\theta}^t(s,a))}{n^t(s,a)}\log\frac{12SAT}{\delta}}+\frac{7}{3n^t(s,a)-1}\cdot\log\frac{12SAT}{\delta}\nonumber\\
    &\left.|\sqrt{\theta(1-\theta)}-\sqrt{\hat{\theta}^t(s,a)^t(1-\hat{\theta}^t(s,a)}|\leq
    \frac{2\log(SAT/\delta)}{n^t(s,a)-1}
    \right\}\nonumber,
\end{align}
for any state-action pair $(s,a)\in\mathcal{S}\times\mathcal{A}$ and $t\leq T$. Lemma \ref{lem:alg2hp}, which we will introduce later, states that the true adherence level $\theta^t(s,a)$ is also in this set with high probability, and thus, $\bar{\theta}^t(s,a)$ is an upper bound for the true adherence level for all $s,a,h$ and $t\leq T$.

We point out that our algorithm needs to have a pre-determined episode upper bound $T$ as the input. The reason is that Algorithm \ref{alg:alg1} updates the estimation of the rewards and the adherence levels at the end of all episodes, which requires us to take a union probability bound from Hoeffding's inequality for $T$ episodes. In order to alleviate the effect of $T$ in the probability union bound, we need to consider $T$ in the design of the algorithm.

\subsubsection{Algorithm Analysis}
\label{sec:A.2.aa}
In this section, we prove Theorem \ref{thm:UCB-AD} for Algorithm \ref{alg:alg1} under $\mathcal{E}_1$, where the human policy is known while the adherence level is unknown. While the proof is inspired by the proof of Theorem 1 in \cite{dann2015sample}, our proof can achieve a better sample complexity bound than theirs in our setting, where the transition probability $p_{h}^{\MM}$ is known but depends on the time horizon for all $h\in[H]$, and the adherence level $\theta$ is unknown but independent of $h$. Specifically, applying Theorem 1 in \cite{dann2015sample} can only achieve an $O(H^3S^2A/\epsilon^2)$ sample complexity bound. On the contrary, we establish an $O(H^2S^2A/\epsilon^2)$ sample complexity bound, which improves the order of the length of the time horizon in the bound.

First, we seek to build an analysis framework for a simpler problem: the reward function is deterministic. This is because if the reward function is stochastic, we can adopt the idea used in the proof of Lemma \ref{lem:ewbdd} and analyze an additional error term for the estimation of the reward. By doing so, we can obtain a similar result with the same order with respect to $H$, $S$, and $A$ but with different universal constants.

Therefore, in the following, we only analyze the case where the reward function of the machine and the human policy is deterministic. Then, the challenge lies in estimating the adherence levels. Recall that the estimator for the adherence level is
\begin{align*}
    \hat{\theta}^t(s,a)=\frac{1}{n^{t}(s,a)}\sum\limits_{i=1}^{t}\sum\limits_{h=1}^{H}\mathbb{I}(s_h^i=s,a_h^{\MM,i}=a,a^{\HH,i}_h=a),    
\end{align*}
where $n^t(s,a)=\sum\limits_{i=1}^{t}\sum\limits_{h=1}^{H}\mathbb{I}(s_h^i=s,a_h^{\MM,i}=a)$ for all $s\in\mathcal{S}$, $a\in\mathcal{A}$ and $t\leq T.$

Next, we briefly summarize our proof. In general, we follow a similar proof structure as in Theorem 1 of \cite{dann2015sample}. However, our setting features two differences from theirs as listed below. 
\paragraph{Higher updating frequency.}\

The first difference is that we update the estimation of the transition kernels, or the adherence levels at the end of each episode for all states and actions, while \cite{dann2015sample} only update the estimation once after several episodes for only one state-action pair. Although a slower updating rule makes their algorithm slow, the problem of frequent updates is that we are exposed to a larger probability of observing some outliers such that the true adherence level might not be captured by the confidence set \eqref{set:alg2_theta}. Therefore, the probability bound developed by Lemma 1 in \cite{dann2015sample} will fail in our setting.
\paragraph{Non-Stationarity.}\ 

Our setting is non-stationary in the sense that the transition probability $p_{h}^{\MM}$ is different for different $h=1,...,H$ even if the adherence level is independent of the horizon. However, \cite{dann2015sample} consider a stationary setting where the transition kernels depend only on the state-action pairs and are the same for all different $h=1,...,H$. As a result, we cannot directly apply their Lemma C.5 to bound several parameters (those two parameters are $c_1(s,a)$ and $c_2(s,a)$ on page 20 in \cite{dann2015sample}) directly.

To address those two problems, we show the following two lemmas. For the first problem, Lemma \ref{lem:alg2hp} states that even with higher frequent updates, the confidence set \eqref{set:alg2_theta} can still cover the true adherence level with high probability, and thus, Proposition \ref{prop_monotone} implies that our estimated $Q$-values are upper bounds for the corresponding true $Q$-values for all empirical optimal policy $\hat{\pi}^t$ at all episodes $t\leq T$ for any pre-determined constant $T$ in Algorithm \ref{alg:alg1}.
\begin{lemma}
    \label{lem:alg2hp}
    For any fixed $T\geq0$, with probability at least $1-\frac{\delta}{2}$, the set \eqref{set:alg2_theta} contains the true adherence level $\theta(s,a)$ for any state-action pairs $(s,a)\in\mathcal{S}\times{\mathcal{A}}$ and episode $t\leq T$.
\end{lemma}
For the second problem, Lemma \ref{lem:station} transfers the concentration of the non-stationary transition kernels to that of the stationary adherence levels that are independent of the horizon $H$ and, establishes a Bernstein-like bound for all non-stationary transition kernels.  
\begin{lemma}
    \label{lem:station}
    The following inequalities
    \begin{align*}
        \left| 
            p_h^{\MM}(s'|s,a)-\bar{p}^{\MM,t}_h(s'|s,a)
        \right|
        &\leq
        \left| \left( \pi_h(s'|s,a)-\frac{\sum\limits_{a'\not=a}\pi_h(s'|s,a')\pi_h^{\HH}(a'|s)}{1-\pi_{h}^{\HH}(a|s)}\right)(\theta(s,a)-\hat{\theta}^t(s,a))\right|\\
        &\leq
        \sqrt{\frac{8\hat{p}^{\MM,t}_h(s'|s,a)(1-\hat{p}^{\MM,t}_h(s'|s,a))}{n^t(s,a)}\log(12SAT/\delta)}+\frac{26}{3n^t(s,a)-3}\log(12SAT/\delta)
    \end{align*}
    hold with probability no less than $1-\delta$ for all $s,s'\in\mathcal{S}$, $a\in\mathcal{A}$ and $t\leq T$. 
\end{lemma}
Then, based on those two above lemmas, the proof can be shown in a similar approach as in \cite{dann2015sample}.

Before we prove the main theorem, we need some notations. As in \cite{dann2015sample}, denote $w_{t}(s, a)$ the expected visitation frequency of the $(s, a)$-pair under policy $\pi^{t}$, i.e.,
$$
w_{t}(s, a)=\sum_{h=1}^{H} P\left(s_{h}=s, \pi_{h}^{t}\left(s_{h}\right)=a\right).
$$
Next, we denote $\iota_{t}(s, a)$ the importance of $(s,a)$: its relative weight compared to $w_{\min }:=\frac{\epsilon}{4 H|\mathcal{S}|}$ on a log-scale
$$
\iota_{t}(s, a):=\min \left\{z_{i}: z_{i} \geq \frac{w_{t}(s, a)}{w_{\min }}\right\}\in\{0,1,2,4,8, \ldots\} \quad \text { where } z_{1}=0 \text { and } z_{i}=2^{i-2} \,\,\,\,\,\forall i=2,3, \ldots
$$
Intuitively, $\iota_{t}(s, a)$ is an integer indicating the influence of the state-action pair on the value function of $\pi^{t}$. Similarly, we define the knownness
$$
\kappa_{t}(s, a):=\max \left\{z_{i}: z_{i} \leq \frac{n_{t}(s, a)}{m w_{t}(s, a)}\right\} \in\{0,1,2,4, \ldots\},
$$
which indicates how often $(s, a)$ has been observed relative to its importance. The constant $m$ is defined as
\[
    m = 512(\log_2 \log_2 H)^2 \frac{CH2}{\epsilon^2} \log^2_2\left(\frac{8H^2S^2}{\epsilon}\right) \log\left(
         \frac{6CSA\log_2^2(4S^2H2/\epsilon)}{\delta}
    \right) ,  
\]
where $C=\max_{s\in\mathcal{S},a\in\bar{\mathcal{A}}} C(s,a)$, and $C_(s,a)$ denotes possible successor states of state $s$ and action $a$ for $s\in\mathcal{S}$ and $a\in\mathcal{A}$. Thus, we also have $C\leq S$. We can now categorize $(s, a)$-pairs into subsets
$$
X_{t, \kappa, \iota}:=\left\{(s, a) \in X_{t}: \kappa_{t}(s, a)=\kappa, \iota_{t}(s, a)=\iota\right\} \quad \text { and } \quad \bar{X}_{t}=\mathcal{S} \times \mathcal{A} \backslash X_{t}
$$

where $X_{t}=\left\{(s, a) \in \mathcal{S} \times \mathcal{A}: \iota_{t}(s, a)>0\right\}$ is the active set and $\bar{X}_{t}$ the set of state-action pairs that are very unlikely under the current policy.

The proof can be summarized in a few steps:

\begin{enumerate}
  \item The true MDP is in the confidence set of MDPs (those with adherence level in \eqref{set:alg2_theta}) for all episodes $t<T$ with probability at least $1-\delta/2$ (we can ensure this property by Lemma \ref{lem:alg2hp}).

  \item In every episode $t$, the optimistic $Q$-functions $\hat{V}^{\hat{\pi}^t}(\cdot|\bar{\theta}^t)$ is higher than $V^{*}(\cdot|{\theta})$ at least $1-\delta / 2$, which is ensured by Proposition \ref{prop_monotone}.

  \item If  $m=\tilde{\Omega}\left(\frac{H^{2}}{\epsilon} \ln \frac{|\mathcal{S}|}{\delta}\right)$ (which is true by our definition of $m$), the number of episodes with $\left|X_{t, \kappa, \iota}\right|>\kappa$ for some $\kappa$ and $\iota$ are bounded by $\tilde{O}(|\mathcal{S} \times \mathcal{A}| m)$ with probability at least $1-\delta / 2$. To show this step, we can apply Lemma 2 in \cite{dann2015sample}.

  \item If $\left|X_{t, \kappa, \iota}\right| \leq \kappa$ for all $\kappa$, $\iota$, i.e., relevant state-action pairs are sufficiently known and $m=$ $\tilde{\Omega}\left(\frac{C H^{2}}{\epsilon^{2}} \ln \frac{1}{\delta_{1}}\right)$, then the optimistic values $\hat{V}^{\hat{\pi}^t}(s_1|\bar{\theta}^t)$ and ${V}^{\hat{\pi}^t}(s_1|{\theta}^t)$ are $\epsilon$-close to the true MDP value. Together with part 2, we get that with high probability, the policy $\hat{\pi}^{t}$ is $\epsilon$-optimal in this case. To prove this, we can use our Lemma \ref{lem:station} combined with Lemma 3 in \cite{dann2015sample} 

  \item From parts 3 and 4 , we can show that with probability $1-\delta$, at most $\tilde{O}\left(\frac{C|\mathcal{S} \times \mathcal{A}| H^{2}}{\epsilon^{2}} \ln \frac{1}{\delta}\right)=\tilde{O}\left(\frac{H^{2}S^2A}{\epsilon^{2}} \ln \frac{1}{\delta}\right)$ episodes are not $\epsilon$-optimal.
\end{enumerate}
In the following, we show the proof of Theorem \ref{thm:UCB-AD}, Lemmas \ref{lem:alg2hp} and \ref{lem:station}.

\subsubsection{Proof related to Theorem \ref{thm:UCB-AD}}
\label{sec:A.2.pf}
In this section, we first prove Theorem \ref{thm:UCB-AD}, and then prove Lemmas \ref{lem:alg2hp} and \ref{lem:station} that are used in the proof.

\begin{proof}[Proof of Theorem \ref{thm:UCB-AD}]
        In our case, the statements of Lemmas 2 and 3 in \cite{dann2015sample} still hold. Lemma 2 can be shown without any modifications, and for their Lemma 3, we can show the same statement by applying Lemma \ref{lem:station} to establish $c_1(s,a)$ and $c_2(s,a)$ on page 20 of \cite{dann2015sample} for our non-stationary case.
        
        Then, by Lemma 2 in \cite{dann2015sample}, we can bound the number of episodes satisfying $\left|X_{t,\kappa, \iota}\right|>\kappa$ for some $\kappa, \iota$ by $6mSA\cdot \log_2\frac{4H^2S}{\epsilon} \log_2 S =O(C\cdot H^2SA/\epsilon^2)$ with probability at least $1-\delta / 2$. Their Lemma 3 states that 
        \begin{align}
            \label{ieq:lem3}
            \left|V^{\hat{\pi}_{t}}(s_1)-\hat{V}^{\hat{\pi}_{t}}(s_1)\right|<\epsilon
        \end{align}
         for all other episodes. In addition, combining Lemma \ref{lem:alg2hp} and Proposition \ref{prop_monotone}, we know that 
         \begin{align}
            \label{ieq:mono}
             \hat{V}^{\hat{\pi}^{t}}\geq {V}^{\pi^*}\geq {V}^{\hat{\pi}^{t}} 
         \end{align} 
         holds with probability at least $1-\delta / 2$. Thus, we can draw the conclusion that $1-\delta$, with at least $T-O(C\cdot H^2SA/\epsilon^2)$ episodes, the corresponding policy $\hat{\pi}^t$ satisfies 
        \[
        {V}^{\hat{\pi}^{t}}(s_1)+\epsilon\geq\hat{V}^{\hat{\pi}^{t}}(s_1)\geq {V}^{\pi^*}(s_1) \geq {V}^{\hat{\pi}^{t}}(s_1),
        \]
        which implies that the corresponding $\hat{\pi}^t$ is $\epsilon$-optimal. Here, the first inequality comes from \eqref{ieq:lem3}, and others come from \eqref{ieq:mono}. Moreover, recall that $C$ is the maximum number of possible successor states from any state and action pair, implying $C\leq S$. Thus, at most $O(C\cdot H^2SA/\epsilon^2)=O( H^2S^2A/\epsilon^2)$ episodes are not $\epsilon$-optimal, as our statement of Theorem \ref{thm:UCB-AD}. 
\end{proof}

\begin{proof}[Proof of Lemma \ref{lem:alg2hp}]
    The proof is similar to the proof of Lemma 1 in \cite{dann2015sample}. Given the total visiting number $n^t(s,a)$ of a state-action pair $(s,a)$ and the corresponding visiting horizons and episodes $\{(h_l,t_l)\}_{l=1}^{n^t(s,a)}$ such that $s_h^t=s$ and $a_h^{\MM,i}=a$, we have 
    \begin{align}
        \label{eq:mean_theta}
        \mathbb{E}\left[
            \frac{1}{n^t(s,a)}\sum\limits_{l=1}^{n^t(s,a)} \mathbb{I}(s_{h_l}^{t_l}=s,a_{h_l}^{\MM,t_l}=a,a_h^{\HH,t_l}=a)
        \right]
        =
        \theta(s,a)
    \end{align}
    by the definition of $\theta(s,a)$ (recall the definition $n^t(s,a)=\sum\limits_{i=1}^{t}\sum\limits_{h=1}^{H}\mathbb{I}(s_h^i=s,a_h^{\MM,i}=a)$). Then, by the Azuma–Hoeffding's inequality, with given $n^t(s,a)$ and  $\{(h_l,t_l)\}_{l=1}^{n^t(s,a)}$, the following inequality holds with probability no less than $1- \frac{\delta}{12SAT}$
    \begin{align}
    \label{ieq:dtheta1}
        \left\vert\hat{\theta}^t(s,a)-{\theta}(s,a)\right\vert
        =
        \left\vert\frac{1}{n^t(s,a)}\sum\limits_{i=1}^{t}\sum\limits_{h=1}^{H} \mathbb{I}(s_h^i=s,a_h^{\MM,i}=a,a_h^{\HH,i}=a)-{\theta}(s,a)\right\vert
        \leq
        2\sqrt{\frac{\log (12SAT/\delta)}{n^t(s,a)}}
    \end{align}
    holds for all $s\in\mathcal{S}$, $a\in\mathcal{A}$ and $t$. Here, the first step comes from the definition of $\hat{\theta}^t(s,a)$, and the second line comes from Hoeffding's inequality and \eqref{eq:mean_theta}. Consequently, taking a union bound for all $s,a$ and $t\leq T$, we have that with probability no less than $1-\frac{\delta}{12}$
    \begin{align*}
    \label{ieq:dtheta2}
        \left\vert\hat{\theta}^t(s,a)-{\theta}(s,a)\right\vert
        \leq
        2\sqrt{\frac{\log (12SAT/\delta)}{n^t(s,a)}}
    \end{align*}
    holds for all $(s,a)\in\mathcal{S}\times\mathcal{A}$ at the end of any episode $t\leq T$. Similarly, by applying Bernstein’s inequality and taking union bound, we have
    \begin{align}
        \left\vert\hat{\theta}^t(s,a)-{\theta}(s,a)\right\vert
        \leq
        \sqrt{\frac{2{\theta}(s,a)(1-{\theta}(s,a))\log (12SAT/\delta)}{\sqrt{n^t(s,a)}}}
        +\frac{1}{3n^t(s,a)}\log(6SAT/\delta),
    \end{align}    
    which holds also for all $s,a$ and $t\leq T$ with probability no less than  $1-\frac{\delta}{12}$.
    In addition, by applying Theorem 10 in \cite{dann2015sample}, we can have that with probability no less than $1-\frac{\delta}{12HST}$,
    \begin{align}
        \label{ieq:dtheta3}
        \left\vert\sqrt{{\theta}(s,a)(1-{\theta}(s,a))} - 
        \sqrt{\hat{\theta}^t(s,a)(1-\hat{\theta}^t(s,a))}
        \right\vert
        \leq
        \sqrt{\frac{2\log(12SAT)}{n^t(s,a)}}.
    \end{align}
    Thus, combining the three probability bounds \eqref{ieq:dtheta1}, \eqref{ieq:dtheta2}, and \eqref{ieq:dtheta3}, we finally have that with probability no less than $1-\frac{\delta}{4}$, the true adherence level $\theta(s,a)$ satisfies all three inequalities in \eqref{set:alg2_theta} for all $s,a$ and $t\leq T$. 

    Additionally, we remark that our construction in Algorithm \ref{alg:alg1} is compatible with the setting where the reward function $r^{\MM}$ is unknown and can be random. Even if the reward function is random, we can follow the same proof as for the adherence level to show that the estimated average reward functions in Algorithm \ref{alg:alg1} is an upper bound of the true average reward function with probability no less than $1-\frac{\delta}{4}$, and the difference between them is in the order of $O(1/\sqrt{n^t(s,a)})$ for all $s,a$ and $t\leq T$. Therefore, with probability no less than $1-\frac{\delta}{2}$, our estimated reward function and the adherence level are larger than the corresponding true values, and the convergence rate is $O(1/\sqrt{n^t(s,a)})$ for all $s,a$ and $t\leq T$.
\end{proof}

\begin{proof}[Proof of Lemma \ref{lem:station}]
    By applying the definition of $\p_h^{\MM}$ and $\p_h^{\MM,t}$, we have
    \begin{align}
        \label{eq:pMMs}
        p^\MM_h(s'|s,a)
        =
        \left( \pi_h(s'|s,a)-\frac{\sum\limits_{a'\not=a}\pi_h(s'|s,a')\pi_h^{\HH}(a'|s)}{1-\pi_{h}^{\HH}(a|s)}\right)\theta(s,a)+\frac{\sum\limits_{a'\not=a}\pi_h(s'|s,a')\pi_h^{\HH}(a'|s)}{1-\pi_{h}^{\HH}(a|s)},\\
        \hat{p}^{\MM,t}_h(s'|s,a)
        =
        \left( \pi_h(s'|s,a)-\frac{\sum\limits_{a'\not=a}\pi_h(s'|s,a')\pi_h^{\HH}(a'|s)}{1-\pi_{h}^{\HH}(a|s)}\right)\bar{\theta}^t(s,a)+\frac{\sum\limits_{a'\not=a}\pi_h(s'|s,a')\pi_h^{\HH}(a'|s)}{1-\pi_{h}^{\HH}(a|s)}.\nonumber
    \end{align}
    Then, inequalities in \eqref{eq:pMMs} imply the first inequality in Lemma \ref{lem:station} for all $s',s\in\mathcal{S}$, $a\in\mathcal{A}$ and $t$.

    Next, we prove the second inequality. For reading convenience, we let $\zeta_1$ and $\zeta_2$ be
    \[
        \zeta_1= \left( \pi_h(s'|s,a)-\frac{\sum\limits_{a'\not=a}\pi_h(s'|s,a')\pi_h^{\HH}(a'|s)}{1-\pi_{h}^{\HH}(a|s)}\right),\ 
        \zeta_2=\frac{\sum\limits_{a'\not=a}\pi_h(s'|s,a')\pi_h^{\HH}(a'|s)}{1-\pi_{h}^{\HH}(a|s)},
    \]
    for any fixed state-action pair $s,a$ and fixed $t$. Here, by definition, we have $\zeta_1\in[-1,1]$ and $\zeta_2\in[0,1]$. Then, we have
    \begin{align*}
        \left|p^\MM_h(s'|s,a)-\hat{p}^{\MM,t}_h(s'|s,a)\right|
        =
        |\zeta_1||\theta(s,a)-\bar{\theta}^t(s,a)|.
    \end{align*}
    Now, we first show 
    \begin{align}
        \label{ieq:ptheta}
        |\zeta_1|\bar{\theta}^t(s,a)(1-\bar{\theta}^t(s,a))
        \leq
        \hat{p}^{\MM,t}_h(s'|s,a)(1-\hat{p}^{\MM,t}_h(s'|s,a)).
    \end{align}
    If $\zeta_1\geq0$, \eqref{ieq:ptheta} is equivalent to 
    \begin{align*}
        (\zeta_1^2-\zeta_1)(\bar{\theta}^t(s,a))^2+2\zeta_1\zeta_2\bar{\theta}^t(s,a)
        +
        \zeta_2^2-\zeta_1\leq 0,
    \end{align*}
    which can be obtained by checking the non-positivity of the discriminant for the quadratic equation. Specifically, the discriminant is
    \[
        -4\zeta_1-4\zeta_2+4\zeta_1\zeta_2^2+4\zeta_1^2\zeta_2,
    \]
    which is no more than 0 since $\zeta_1,\zeta_2\leq 1$. If $\zeta_1\leq0$, we can similarly prove that the discriminant is still non-positive. Therefore, \eqref{ieq:ptheta} holds. Then, by Lemma C.5 in \cite{dann2015sample}, we have on the event of Lemma \ref{lem:alg2hp}, 
    \begin{align}
        \label{ieq:mult_theta}
        |\theta(s,a)-\bar{\theta}^t(s,a)|
        \leq
        \sqrt{\frac{8\bar{\theta}^t(s,a)(1-\bar{\theta}^t(s,a))}{n^t(s,a)}\log(12SAT/\delta)}+\frac{26}{3n^t(s,a)-3}\log(12SAT/\delta).
    \end{align}
    Combining \eqref{ieq:mult_theta} and \eqref{ieq:ptheta}, and plugging them into \eqref{eq:pMMs}, we arrive to
    \begin{align*}
        \left|p^\MM_h(s'|s,a)-\hat{p}^{\MM,t}_h(s'|s,a)\right|
        &=
        |\zeta_1||\theta(s,a)-\bar{\theta}^t(s,a)|\\
        &\leq 
        |\zeta_1|\sqrt{\frac{8\bar{\theta}^t(s,a)(1-\bar{\theta}^t(s,a))}{n^t(s,a)}\log(12SAT/\delta)}+\frac{26|\zeta_1|}{3n^t(s,a)-3}\log(12SAT/\delta)\\
        &\leq
        \sqrt{\frac{8|\zeta_1|\bar{\theta}^t(s,a)(1-\bar{\theta}^t(s,a))}{n^t(s,a)}\log(12SAT/\delta)}+\frac{26}{3n^t(s,a)-3}\log(12SAT/\delta)\\
        &\leq
        \sqrt{\frac{8 \hat{p}^{\MM,t}_h(s'|s,a)(1-\hat{p}^{\MM,t}_h(s'|s,a))}{n^t(s,a)}\log(12SAT/\delta)}+\frac{26}{3n^t(s,a)-3}\log(12SAT/\delta),
    \end{align*}
    where the first line comes from \eqref{eq:pMMs}, the second line comes from \eqref{ieq:ptheta}, the third line comes from the fact that $|\zeta_1|\leq1$, and the last line comes from \eqref{ieq:mult_theta}. The proof of this inequality holds for all $s\in\mathcal{S}$, $a\in\mathcal{A}$ and $t\leq T$ on the event of Lemma \ref{lem:alg2hp}, which holds with probability no less than $1-\delta$. Thus, we finish the proof.
\end{proof}


\subsection{Supplementary Materials for Algorithm \ref{alg:alg2}}
\label{ap_pf_alg2}
In this section, we first provide definitions of the estimations of the transition kernels and rewards in Algorithm \ref{alg:alg2}, which can be found in Section \ref{apnd:A.1.def}. In Section \ref{apnd:A.1.thm}, we will present the proof for Theorem \ref{thm_RFE} and Corollary \ref{cor_CMDP}. In section \ref{apnd:A.1.eslem}, we prove important ancillary lemmas for Theorem \ref{thm_RFE}.

\subsubsection{Definitions in Algorithm \ref{alg:alg2}}
\label{apnd:A.1.def}
We first formally define the empirical estimation $\hat{p}^{\MM,t}_h$ in Algorithm \ref{alg:alg2} as follows:
\begin{equation*}
    \begin{aligned}
        \hat{p}^{\MM,t}_h(s'|s,a) &= \cfrac{n_h^t(s, a, s')}{n_h^t(s, a)} \text{ \,\,\,if $n_h^t(s, a) > 0$\,\,\,\,\,\,and\,\,\,\,\,\,} \hat{p}^{\MM,t}_h(s'|s,a) = \cfrac{1}{S} \text{ \,\,\, otherwise},
    \end{aligned}
\end{equation*}
where $n_h^t(s,a) = \sum_{i=1}^t\mathbb{I}{\left\{\left(s_h^i, a_h^{\MM, i}\right) = (s, a)\right\}}$ is the number of times in the first $t$ episodes at the time step $h$, state $s$, and the machine gives advice $a$; $n_h^t(s,a, s') = \sum_{i=1}^t\mathbb{I}{\left\{\left(s_h^i, a_h^{\MM, i}, s_{h+1}^i\right) = (s, a,s')\right\}}$ is the number of times at time $h$, state $s$, the machine gives advice $a$, and reached state $s'$ at time $h+1$ in the first $t$ episode. Similarly, the empirical reward is defined as 
\begin{equation*}
    \begin{aligned}
        \hat{r}^{\MM,t}_h(s,a) &= \cfrac{\sum_{i=1}^tr^i(s,a)\mathbb{I}{\left\{\left(s_h^i, a_h^{\MM, i}\right) = (s, a)\right\}}}{n_h^t(s, a)} \text{ \,\,\,if $n_h^t(s, a) > 0$\,\,\,\,\,\,and\,\,\,\,\,\,} \hat{r}^{\MM,t}_h(s,a) = 0 \text{ \,\,\, otherwise}.
    \end{aligned}
\end{equation*}

\subsubsection{Proof of Theorem \ref{thm_RFE} and Corollary \ref{cor_CMDP}}
\label{apnd:A.1.thm}
    To prove Theorem \ref{thm_RFE}, we first state the algorithm and develop the corresponding sample complexity bound (with probability $1-\delta$) for the no-penalty case, where the reward falls within the interval $[0,1]$. Next, we show Theorem \ref{thm_RFE} by extending the above results into the case with the penalty $\beta\in(0,H)$, which can be addressed by scaling the penalized reward in the range $(0, H)$ to the range $[0,1]$.

    To solve the no-penalty case, as mentioned in the remark after Theorem \ref{thm_RFE}, we change THRESHOLD from $\epsilon/H$ to $\epsilon/2$. The algorithm is summarized in Algorithm \ref{alg:alg3}, and we define $\tau_1$ as the corresponding stopping time. In the following, we characterize the upper bound of $\tau_1$ and $V_1^*(s_1) - V_1^{\hat{\pi}^{\tau_1}}(s_1)$ when there is no penalty.
    
    \begin{algorithm}[ht!]
    \caption{: RFE-ADvice}
    \label{alg:alg3}
    \begin{algorithmic}[1]
    \State Input: $\epsilon, \delta$
    \State \textbf{Stage 1: Reward-free exploration}
    \State Initialize $t = 1$, THRESHOLD = $\epsilon/2$, and $W^t_h(s,a) = H$ for all $(s,a) \in \mathcal{S}\times\mathcal{A}$ 
    \State Compute ${\pi}^t$ so that ${\pi}_h^t(s) = \argmax_{a\in\mathcal{A}}W_h^t(s, a)$ (see \eqref{eqn_W})
    \While{$W_1^t(s_1, \pi^t(s_1))+4e\sqrt{W_1^t(s_1, \pi^t(s_1))} > \text{THRESHOLD}$}
    \State Sample trajectory $z^t = \{s_1^t, a_1^{\MM,t}, a_1^{\HH,t},r_1^t, \cdots, s_H^t, a_H^{\MM,t}, a_H^{\HH,t}, r_H^t\}$ following $\pi^t$
    \State update $t \leftarrow t+1$, $\mathcal{D} \leftarrow \mathcal{D}\cup \{z^t\}$, $\hat{p}^{\MM,t}_h(s'|s,a)$, $\hat{r}^{\MM,t}_h(s,a)$, and $W_h^t(s,a)$
    \EndWhile
    \State \textbf{Stage 2: Policy identification}
    \State Use planning algorithms to output optimal advice policy $\hat{\pi}^{\tau_1}$ for $\left(\mathcal{S},\bar{\mathcal{A}},H,\hat{p}^{\MM}, \hat{r}^{\MM}\right)$
    \end{algorithmic}
    \end{algorithm}

\paragraph{Upper bounds of $\tau_1$ and $V_1^*(s_1) - V_1^{\hat{\pi}^{\tau_1}}(s_1)$.}\

To establish the upper bound for $\tau_1$, our approach is similar to the proof of Theorem 1 in \cite{menard2021fast}. First notice that 
$$
\begin{aligned}
V_{1}^{\star}\left(s_{1}\right)-V_{1}^{\widehat{\pi}^{\tau_1}}\left(s_{1} \right) & =V_{1}^{\star}\left(s_{1}\right)-\widehat{V}_{1}^{\tau_1, \pi^{\star}}\left(s_{1}\right)+\widehat{V}_{1}^{\tau_1, \pi^{\star}}\left(s_{1}\right)-\widehat{V}_{1}^{\tau_1, \widehat{\pi}^{\tau_1}}\left(s_{1}\right)+\widehat{V}_{1}^{\tau_1, \widehat{\pi}^{\tau_1}}\left(s_{1}\right)-V_{1}^{\widehat{\pi}^{\tau_1}}\left(s_{1}\right) \\
& \leq\left|V_{1}^{\star}\left(s_{1}\right)-\widehat{V}_{1}^{\tau_1, \pi^{\star}}\left(s_{1}\right)\right|+\left|\widehat{V}_{1}^{\tau_1, \widehat{\pi}^{\tau_1}}\left(s_{1}\right)-V_{1}^{\widehat{\pi}^{\tau_1}}\left(s_{1}\right) \right|,\end{aligned}
$$
where the second inequality is because $\widehat{V}_{1}^{\tau_1, \pi^{\star}}\left(s_{1}\right)-\widehat{V}_{1}^{\tau_1, \widehat{\pi}^{\tau_1}}\left(s_{1}\right) \leq 0$. Therefore, we need to show that the empirical MDP is close to the original MDP so that the value function of the same policy is bounded by $\epsilon/2$. To motivate this, recall the Bellman equation of the true $Q$-values and the empirical:
    \begin{align}
        \label{def:Qval}
        Q_h^{\pi}(s,a) = r^{\MM}_h(s,a) + \sum_{s'\in \mathcal{S}}p^{\MM}_h(s'|s,a) Q_{h+1}^{\pi}(s',\pi_{h+1}(s')), \\
        \hat{Q}_h^{t,\pi}(s,a) = \hat{r}^{\MM,t}_h(s,a) + \sum_{s'\in \mathcal{S}}\hat{p}^{\MM,t}_h(s'|s,a) \hat{Q}_{h+1}^{t,\pi}(s',\pi_{h+1}(s'))\nonumber,
    \end{align}
    where $Q_{H+1}^{\pi}(s,a)=\hat{Q}_{H+1}^{t,\pi}(s,a)=0$ for all staet $s$, action $a$, episode $t$, and policy $\pi$.
    
    Denote $\hat{e}_{h}^{t,\pi}(s,a;r)=|\hat{Q}^{t,\pi}_{h}(s,a;\hat{r})-{Q}^{\pi}_{h}(s,a;r)|$ the difference between the empirical and the real $Q$-value for the machine with respect to \textbf{any} policy $\pi$, state $s$, action $a$, reward $r$ ($\hat{r}$ is the sample estimation of $r$ in episode $t$), and horizon $h$ at the $t$-th episode, where the empirical $Q$-value is evaluated by the estimated transition kernels $\hat{p}_{h}^{\MM,t}$ and reward function $\hat{r}_{h}^{t}$ at the $t$-th episode. We immediately have the following bound
    $$\left|V_{1}^{\star}\left(s_{1}\right)-\widehat{V}_{1}^{\tau_1, \pi^{\star}}\left(s_{1}\right)\right|+\left|\widehat{V}_{1}^{\tau_1, \widehat{\pi}^{\tau_1}}\left(s_{1}\right)-V_{1}^{\widehat{\pi}^{\tau_1}}\left(s_{1}\right) \right| \leq \hat{e}_{1}^{t,\pi^*}(s_1,\pi^*(s_1);r^{\MM})+\hat{e}_{1}^{t,\hat{\pi}^{\tau_1}}(s_1,\hat{\pi}^{\tau_1}(s_1);r^{\MM}).$$ 
    
    Notice that here, our definition of $\hat{e}_{h}^{t,\pi}(s,a;r)$ incorporates random reward functions, and is different from that of \cite{menard2021fast}. Next, we show the following uniform bound for $\hat{e}_{1}^{t,\pi}(s_1,\pi(s_1);r)$.
\begin{lemma}
    \label{lem:ewbdd}
    With probability at least $1-\delta$, for any episode $t$, policy $\pi$, and reward function $r$ that is in $[0,1]$,
    \[
        \hat{e}_{1}^{t,\pi}(s_1,\pi(s_1);r)\leq 4e\sqrt{\max_{a\in\mathcal{A}}W_{1}^{t}(s_1,a)}+\max_{a\in\mathcal{A}}W_{1}^{t}(s_1,a).
    \]
\end{lemma}
Recall $W^t_1(\cdot,\cdot)$ is defined by equation \eqref{eqn_W} and it is a function of $\delta$ (we omit the dependence of $\delta$ in $W$ for notation simplicity). With Lemma \ref{lem:ewbdd}, we know that for our quantity of interest, we just need to bound $4e\sqrt{\max_{a\in\mathcal{A}}W_{1}^{t}(s_1,a)}+\max_{a\in\mathcal{A}}W_{1}^{t}(s_1,a)$ with the following lemma.

\begin{lemma} \label{lem_tau1}
For $\epsilon > 0$ and $\delta > 0$, with probability at least $1-\delta$, we have 
$$4e\sqrt{\max_{a\in\mathcal{A}}W_{1}^{\tau_1}(s_1,a)}+\max_{a\in\mathcal{A}}W_{1}^{\tau_1}(s_1,a) \leq \epsilon/2$$
and the terminating time $\tau_1$ is bounded by
$$
\tau_1 \leq \frac{H^{3} S A}{\varepsilon^{2}}(\log (4 S A H / \delta)+S) C_{1}+1,
$$
where $C_{1} = 9000e^6\log^2\left(e^{18}\left(\log(4HSAT/\delta)+S\right)\frac{H^4SA}{\epsilon}\right).$
\end{lemma}
With Lemma \ref{lem_tau1}, we know that with probability $1-\delta$, 
$$
\begin{aligned}
V_{1}^{\star}\left(s_{1}\right)-V_{1}^{\widehat{\pi}^{\tau_1}}\left(s_{1} \right) & \leq\left|V_{1}^{\star}\left(s_{1}\right)-\widehat{V}_{1}^{\tau_1, \pi^{\star}}\left(s_{1}\right)\right|+\left|\widehat{V}_{1}^{\tau_1, \widehat{\pi}^{\tau_1}}\left(s_{1}\right)-V_{1}^{\widehat{\pi}^{\tau_1}}\left(s_{1}\right) \right|\\ 
&\leq \hat{e}_{h}^{t,\pi^*}(s_1,\pi^*(s_1);r^{\MM})+\hat{e}_{h}^{t,\hat{\pi}^{\tau_1}}(s_1,\hat{\pi}^{\tau_1}(s_1);r^{\MM})\\
&\leq 2\left(4e\sqrt{\max_{a\in\mathcal{A}}W_{1}^{t}(s_1,a)}+\max_{a\in\mathcal{A}}W_{1}^{t}(s_1,a)\right)\leq \epsilon.\\
\end{aligned}
$$

\paragraph{Upper bounds of $\tau$ and $V_{1,\beta}^*(s_1) - V_{1,\beta}^{\hat{\pi}^{\tau}_{\beta}}(s_1)$.}\

To establish the result for $\tau$, we first scale the reward from $[0,H]$ to $[0,1]$. From Algorithm 2 and Lemma \ref{lem_tau1}, by setting the threshold to $\epsilon/(2H)$, we know that with probability $1-\delta$, $\tau = \tilde{O}(H^3S^2A/(\epsilon^2/ H^2)) = \tilde{O}(H^5S^2A/\epsilon^2)$. The rest is to show that we have for all $\beta \in (0, H)$,
$$V_{1,\beta}^*(s_1) - V_{1,\beta}^{\hat{\pi}^{\tau}_{\beta}}(s_1) \leq \epsilon.$$
With an analysis similar to that of $\tau_1$, for any $\beta \in [0,H)$, when we scale back the reward to $[0,H]$ (notice the multiplier $H$ in the last inequality), we have 
$$
\begin{aligned}
V_{1,\beta}^{\star}\left(s_{1}\right)-V_{1,\beta}^{\widehat{\pi}^{\tau}_{\beta}}\left(s_{1} \right) & =V_{1,\beta}^{\star}\left(s_{1}\right)-\widehat{V}_{1,\beta}^{\tau, \pi_{\beta}^{\star}}\left(s_{1}\right)+\widehat{V}_{1,\beta}^{\tau, \pi_{\beta}^{\star}}\left(s_{1}\right)-\widehat{V}_{1,\beta}^{\tau, {\widehat{\pi}^{\tau}_{\beta}}}\left(s_{1}\right)+\widehat{V}_{1,\beta}^{\tau, {\widehat{\pi}^{\tau}_{\beta}}}\left(s_{1}\right)-V_{1,\beta}^{{\widehat{\pi}^{\tau}_{\beta}}}\left(s_{1}\right) \\
& \leq\left|V_{1,\beta}^{\star}\left(s_{1}\right)-\widehat{V}_{1,\beta}^{\tau, \pi_{\beta}^{\star}}\left(s_{1}\right)\right|+\left|\widehat{V}_{1,\beta}^{\tau, {\widehat{\pi}^{\tau}_{\beta}}}\left(s_{1}\right)-V_{1,\beta}^{{\widehat{\pi}^{\tau}_{\beta}}}\left(s_{1}\right)\right|\\
&\leq \hat{e}_{1}^{\tau,\pi_{\beta}^*}(s_1,\pi_{\beta}^*(s_1);r^{\MM}_{\beta})+\hat{e}_{1}^{\tau,\hat{\pi}_{\beta}^{\tau}}(s_1,\hat{\pi}_{\beta}^{\tau}(s_1);r^{\MM}_{\beta})\\
&\leq 8e\sqrt{\max_{a\in\mathcal{A}}W_{1}^{\tau}(s_1,a)}+2\max_{a\in\mathcal{A}}W_{1}^{\tau}(s_1,a) \leq 2H\epsilon/(2H) = \epsilon.
\end{aligned}
$$
To conclude the proof, we just have to notice that $8e\sqrt{\max_{a\in\mathcal{A}}W_{1}^{\tau}(s_1,a)}+2\max_{a\in\mathcal{A}}W_{1}^{\tau}(s_1,a)$ is a bound for any reward function $r_{\beta}^{\MM}$ with $\beta\in[0,H)$.

\paragraph{Proof for Corollary \ref{cor_CMDP}.}\

Recall that we have the following CMDP and sample CMDP defined as
\begin{equation}\label{eqn_cmdp_ap}
    \begin{aligned}
        \max_{\pi} \,\,\,& \mathbb{E}^{\pi}\left[\sum_{h=1}^H r^{\MM}(s_h, a_h)\right]\,\,\,\,\,\,\,
        s.t.\,\,\, \mathbb{E}^{\pi}\left[\sum_{h=1}^H \mathbb{I}\{a_h \neq \text{defer}\}\right] \leq D,
    \end{aligned}
\end{equation}
\begin{equation}\label{eqn_sample_cmdp_ap}
    \begin{aligned}
        \max_{\pi} \,\,\,\hat{\mathbb{E}}^{\pi}\left[\sum_{h=1}^H \hat{r}^{\MM,\tau}(s_h, a_h)\right]\hspace{5mm}
        s.t.\,\,\, \hat{\mathbb{E}}^{\pi}\left[\sum_{h=1}^H \mathbb{I}\{a_h \neq \text{defer}\}\right] \leq D,
    \end{aligned}
\end{equation}
where $\hat{\mathbb{E}}$ denotes that the transition kernel follows $\hat{p}^{\MM,\tau}$, and $\pi^*_D$ and $\hat{\pi}^{\tau}_D$ are the corresponding solutions for the above CMDP problems.

From the standard primal-dual theorem, there exists non-negative $\beta_D^*$, which is the optimal dual variable for the constraint $\mathbb{E}[\sum_{h=1}^H\mathbb{I}{\{a_h \neq \text{defer}\}}] \leq D$, such that $\beta_D^*$ and the optimal policy for \eqref{eqn_cmdp_ap} solve the following saddle point problem
\begin{equation*}
    \begin{aligned}
        &\hspace{4mm}\min_{\beta \geq 0} \max_{\pi\in\Pi} \left(\mathbb{E}^{\pi}\left[\sum_{h=1}^H r^{\MM}(s_h, a_h)\right] +\beta \left(D - \mathbb{E}^{\pi}\left[\sum_{h=1}^H \mathbb{I}\{a_h \neq \text{defer}\}\right]\right)\right)\\
        &=\min_{\beta \geq 0} \max_{\pi\in\Pi} \left(V^{\pi}_{\beta}(s_1)+\beta D\right).
    \end{aligned}
\end{equation*}
We observe that the optimal policy for \eqref{eqn_cmdp_ap} is the same as the optimal policy that maximize $V_{\beta_{D}^*}^{\pi}$. Therefore, adding penalty $\beta$ for advice and solving $V_{\beta}^*$ are equivalent to imposing constraint to the original MDP and solving the corresponding CMDP problem.

Now let us start the proof. For the proof of constraint violation, from the results for $\tau_1$, we can view $\mathbb{E}^{\pi}\left[\sum_{h=1}^H \mathbb{I}\{a_h \neq \text{defer}\}\right]$ and $\hat{\mathbb{E}}^{\pi}\left[\sum_{h=1}^H \mathbb{I}\{a_h \neq \text{defer}\}\right]$ as value functions under transition kernels $\hat{p}^{\MM,\tau}$ and $p^{\MM}$. Therefore by knowing that
$$\left|\mathbb{E}^{\hat{\pi}^{\tau}_D}\left[\sum_{h=1}^H \mathbb{I}\{a_h \neq \text{defer}\}\right] - \hat{\mathbb{E}}^{\hat{\pi}^{\tau}_D}\left[\sum_{h=1}^H \mathbb{I}\{a_h \neq \text{defer}\}\right]\right| \leq \frac{\epsilon}{H}, \,\,\,\text{ and }\,\,\,\hat{\mathbb{E}}^{\hat{\pi}^{\tau}_D}\left[\sum_{h=1}^H \mathbb{I}\{a_h \neq \text{defer}\}\right] \leq D,$$
we have
$$\mathbb{E}^{\hat{\pi}^{\tau}_D} \left[\sum_{h=1}^H \mathbb{I}\{a_h \neq \text{defer}\}\right] \leq D+\epsilon.$$
To show the $\epsilon$-optimal property of the objective function, by observing that $\mathbb{E}^{\pi}\left[\sum_{h=1}^H r^{\MM}(s_h, a_h)\right] = V_1^{\pi}(s_1)$, we can get the following decomposition 
$$
\begin{aligned}
V_{1}^{\pi^*_D}\left(s_{1}\right)-V_{1}^{\widehat{\pi}^{\tau}_D}\left(s_{1} \right) & =V_{1}^{\pi^*_D}\left(s_{1}\right) -\widehat{V}_{1}^{{\pi^*_D}}\left(s_{1}\right)+\widehat{V}_{1}^{{\pi^*_D}}\left(s_{1}\right) -\widehat{V}_{1}^{\widehat{\pi}^{\tau}_D}\left(s_{1}\right)+\widehat{V}_{1}^{\widehat{\pi}^{\tau}_D}\left(s_{1}\right)-V_{1}^{\widehat{\pi}^{\tau}_D}\left(s_{1}\right) \\
& \leq\left|V_{1}^{\pi^*_D}\left(s_{1}\right) -\widehat{V}_{1}^{{\pi^*_D}}\left(s_{1}\right)\right|+\left|\widehat{V}_{1}^{\widehat{\pi}^{\tau}_D}\left(s_{1}\right)-V_{1}^{\widehat{\pi}^{\tau}_D}\left(s_{1}\right)\right|
+\widehat{V}_{1}^{{\pi^*_D}}\left(s_{1}\right) -\widehat{V}_{1}^{\widehat{\pi}^{\tau}_D}\left(s_{1}\right)\\
&\leq \frac{2\epsilon}{H} + \widehat{V}_{1}^{{\pi^*_D}}\left(s_{1}\right) -\widehat{V}_{1}^{\widehat{\pi}^{\tau}_D}\left(s_{1}\right),
\end{aligned}
$$
where in the last inequality, we have the bound $\epsilon/H$ because the reward is in the scale $(0,1)$ and the terminating time is $\tau$. For the term $\widehat{V}_{1}^{{\pi^*_D}}\left(s_{1}\right) -\widehat{V}_{1}^{\widehat{\pi}^{\tau}_D}\left(s_{1}\right)$, we note that from the primal-dual property, based on the primal problem \eqref{eqn_sample_cmdp_ap}, there exists $\hat{\beta}_D\in[0,H)$ such that $\widehat{\pi}^{\tau}_D = \argmax_{\pi\in\Pi} \hat{V}^{\pi}_{\hat{\beta}_{D}}(s)$. Therefore, we know that 
$$\hat{V}^{\hat{\pi}_D^{\tau}}_{\hat{\beta}_{D}}(s_1) \geq \hat{V}^{\pi^*_D}_{\hat{\beta}_{D}}(s_1),$$
which implies
\begin{equation*}
    \begin{aligned}
        \widehat{V}_{1}^{{\pi^*_D}}\left(s_{1}\right) -\widehat{V}_{1}^{\widehat{\pi}^{\tau}_D}\left(s_{1}\right)\leq \hat{\beta}_D\left(\hat{\mathbb{E}}^{\pi^*_D} \left[\sum_{h=1}^H \mathbb{I}\{a_h \neq \text{defer}\}\right] - \hat{\mathbb{E}}^{\hat{\pi}^{\tau}_D} \left[\sum_{h=1}^H \mathbb{I}\{a_h \neq \text{defer}\}\right]\right).
    \end{aligned}
\end{equation*}
Next, we discuss different cases on $\hat{\beta}_D$. If $\hat{\beta}_D = 0$, this means that $\hat{\mathbb{E}}^{\hat{\pi}^{\tau}_D} \left[\sum_{h=1}^H \mathbb{I}\{a_h \neq \text{defer}\}\right] < D$ and the constrained optimization problem can be treated as an unconstrained one. Therefore we have $\widehat{V}_{1}^{\widehat{\pi}^{\tau}_D}\left(s_{1}\right)$ being the optimal solution for the unconstrained problem, and $\widehat{V}_{1}^{{\pi^*_D}}\left(s_{1}\right) -\widehat{V}_{1}^{\widehat{\pi}^{\tau}_D}\left(s_{1}\right) \leq 0$. For the case where $\hat{\beta}_D \in (0,H)$, we have $\hat{\mathbb{E}}^{\hat{\pi}^{\tau}_D} \left[\sum_{h=1}^H \mathbb{I}\{a_h \neq \text{defer}\}\right] = D$ and 
\begin{equation*}
    \begin{aligned}
        \widehat{V}_{1}^{{\pi^*_D}}\left(s_{1}\right) -\widehat{V}_{1}^{\widehat{\pi}^{\tau}_D}\left(s_{1}\right)\leq \hat{\beta}_D\left(\hat{\mathbb{E}}^{\pi^*_D} \left[\sum_{h=1}^H \mathbb{I}\{a_h \neq \text{defer}\}\right] - D\right).
    \end{aligned}
\end{equation*}
By viewing $\hat{\mathbb{E}}^{\pi^*_D} \left[\sum_{h=1}^H \mathbb{I}\{a_h \neq \text{defer}\}\right]$ as value function, from the property of $\tau$ and the original CMDP \eqref{eqn_cmdp_ap} we know that  
$$\left|\hat{\mathbb{E}}^{\pi^*_D} \left[\sum_{h=1}^H \mathbb{I}\{a_h \neq \text{defer}\}\right] - \mathbb{E}^{\pi^*_D} \left[\sum_{h=1}^H \mathbb{I}\{a_h \neq \text{defer}\}\right]\right| \leq \frac{\epsilon}{H},\,\,\,\text{ and }\,\,\,\mathbb{E}^{\pi^*_D} \left[\sum_{h=1}^H \mathbb{I}\{a_h \neq \text{defer}\}\right] \leq D.$$
Therefore, we have $\hat{\mathbb{E}}^{\pi^*_D} \left[\sum_{h=1}^H \mathbb{I}\{a_h \neq \text{defer}\}\right] \leq D + \epsilon/H$, and $\widehat{V}_{1}^{{\pi^*_D}}\left(s_{1}\right) -\widehat{V}_{1}^{\widehat{\pi}^{\tau}_D}\left(s_{1}\right) \leq H(\epsilon/H) = \epsilon$. 
Finally, we conclude the proof by observing that (using the convention $H > 1$)
$$
\begin{aligned}
V_{1}^{\pi^*_D}\left(s_{1}\right)-V_{1}^{\widehat{\pi}^{\tau}_D}\left(s_{1} \right) \leq \frac{2\epsilon}{H} + \widehat{V}_{1}^{{\pi^*_D}}\left(s_{1}\right) -\widehat{V}_{1}^{\widehat{\pi}^{\tau}_D}\left(s_{1}\right) \leq \frac{2\epsilon}{H} + \epsilon \leq 2\epsilon.
\end{aligned}
$$

\subsubsection{Proof of Essential Lemmas for Theorem \ref{thm_RFE}} \label{apnd:A.1.eslem}
In this section, we list the proof of lemmas for Theorem \ref{thm_RFE}. First, we need to introduce the high-probability event to characterize the high-probability bound.

Denote $p_h^{\MM,\pi}(s,a)$ the probability that the pair $(s,a)$ is visited under the model's transition kernel at time $t$ following policy $\pi$, and $\pi^t$ the policy of Algorithm 2 at episode $t$. We define the pseudo-counts to be 
$$\bar{n}_{h}^{t}(s, a) = \sum_{l=1}^{t} p_{h}^{\MM,\pi^l}(s, a).$$
We define the following events that are favorable: $\mathcal{E}$, the event where the empirical transition probabilities are close to the true ones; $\mathcal{E}^{\text {cnt }}$, the event where the counts are close to pseudo-counts, their expectations; and $\mathcal{E}^r$, the event where the estimation of reward function is close to the expected reward function within sufficient large episodes. More specifically
\begin{equation}
\begin{aligned}
& \mathcal{E} \triangleq\left\{\forall t \in \mathbb{N}, \forall h \in[H], \forall(s, a) \in \mathcal{S} \times \mathcal{A}: \operatorname{KL}\left(\widehat{p}_{h}^{\MM,t}(\cdot \mid s, a), p^{\MM}_{h}(\cdot \mid s, a)\right) \leq \frac{\beta\left(n_{h}^{t}(s, a), \delta\right)}{n_{h}^{t}(s, a)}\right\}, \\
& \mathcal{E}^{\mathrm{cnt}} \triangleq\left\{\forall t \in \mathbb{N}, \forall h \in[H], \forall(s, a) \in \mathcal{S} \times \mathcal{A}: n_{h}^{t}(s, a) \geq \frac{1}{2} \bar{n}_{h}^{t}(s, a)-\beta^{\mathrm{cnt}}(\delta)\right\},\\
&\mathcal{E}^r\triangleq\left\{
    \forall t\leq 4^{10}e^{12} S^{11}A^{11}H^{11}/(\epsilon^{12}\delta^{12}), \forall h \in[H], \forall(s, a) \in \mathcal{S} \times \mathcal{A}: |r_h^{\MM}(s,a)-\hat{r}_h^{\MM}(s,a)|\leq\sqrt{\frac{\beta^{r}\left(n_{h}^{t}(s, a)\right)}{n_h^t(s,a)}}
\right\}
\end{aligned}
\end{equation}
where $\text{KL}(\cdot,\cdot)$ is the KL-divergence of two distributions. We define the $\beta$ functions and show the high probability events in the following lemma
\begin{lemma}
    \label{lem:hp} For the following choices of functions $\beta$,

$$
\begin{aligned}
\phi(n, \delta) & = 6\log (4e S A H / (\epsilon\delta))+S \log (8 e(n+1)), \\
\beta^{\mathrm{cnt}}(\delta) & = 6\log (4e S A H / (\epsilon\delta)),\\
\beta^{r}\left(\delta\right) & = 6\log (4e S A H / (\epsilon\delta)),
\end{aligned}
$$
it holds that
$$
P(\mathcal{E}) \geq 1-\delta/3, \quad P\left(\mathcal{E}^{\mathrm{cnt}}\right) \geq 1-\delta/3, 
\text{ and } 
P\left(\mathcal{E}^{\mathrm{r}}\right)  \geq 1-\delta/3
$$
\end{lemma}

With Lemma \ref{lem:hp} that characterize the high probability event $\mathcal{E}\cap\mathcal{E}^{cnt}\cap\mathcal{E}^{r}$, we are able to prove lemma \ref{lem:ewbdd} and \ref{lem_tau1}.

    \begin{proof}[Proof of Lemma \ref{lem:ewbdd}]
        The proof is similar to the proof of Lemma 1 in \cite{menard2021fast}, and the difference is that we further analyze the estimation error for the rewards while they assume that the reward function is known and deterministic. From lemma \ref{lem:hp}, the event $\mathcal{E}\cap\mathcal{E}^r$ has probability at least $1-\delta$, and in this proof every computation is carried out assuming under $\mathcal{E}\cap\mathcal{E}^r$. For any policy $\pi$, we have
        \begin{align}
            \label{ieq:up_e}
            \hat{e}_{h}^{t,\pi}(s,a;r)
            =&
            |\hat{Q}^{t,\pi}_{h}(s,a;r)-{Q}^{\pi}_{h}(s,a;r)|\nonumber\\
            \leq&
            |r^{\MM}_h(s,a)-\hat{r}^{\MM,t}_h(s,a)|+\left\vert\sum_{s'\in \mathcal{S}}p^{\MM}_h(s'|s,a) Q_{h+1}^{\pi}(s',\pi_{h+1}(s'))-
            \sum_{s'\in \mathcal{S}}\hat{p}^{\MM,t}_h(s'|s,a) \hat{Q}_{h+1}^{t,\pi}(s',\pi_{h+1}(s'))\right\vert,\nonumber\\
            \leq&
            |r^{\MM}_h(s,a)-\hat{r}^{\MM,t}_h(s,a)|
            +
            \left\vert\sum_{s'\in \mathcal{S}}(p^{\MM}_h(s'|s,a) -\hat{p}^{\MM,t}_h(s'|s,a) ) Q_{h+1}^{\pi}(s',\pi_{h+1}(s'))\right\vert\\
            &+\left\vert
            \sum_{s'\in \mathcal{S}}\hat{p}^{\MM,t}_h(s'|s,a)( Q_{h+1}^{\pi}(s',\pi_{h+1}(s'))-\hat{Q}_{h+1}^{t,\pi}(s',\pi_{h+1}(s')))\right\vert\nonumber\\
            \leq&
            |r^{\MM}_h(s,a)-\hat{r}^{\MM,t}_h(s,a)|
            +
            \sum_{s'\in \mathcal{S}}\left\vert p^{\MM}_h(s'|s,a) -\hat{p}^{\MM,t}_h(s'|s,a) \right\vert Q_{h+1}^{\pi}(s',\pi_{h+1}(s'))\nonumber\\
            &+
            \sum_{s'\in \mathcal{S}}\hat{p}^{\MM,t}_h(s'|s,a)\left\vert Q_{h+1}^{\pi}(s',\pi_{h+1}(s'))-\hat{Q}_{h+1}^{t,\pi}(s',\pi_{h+1}(s'))\right\vert,\nonumber\\
            \end{align}
            where the first step comes from the definition of $\hat{e}_{h}^{t,\pi}(s,a;r)$, the second step comes from the recursive definition of the $Q$-values \eqref{def:Qval}, and the last step comes from the triangle inequality for the absolute value. Next, we will apply empirical Bernstein's inequality to the terms above. Specifically,
            \begin{align}
            &\leq |r^{\MM}_h(s,a)-\hat{r}^{\MM,t}_h(s,a)| +  3\sqrt{\frac{\text{Var}_{\hat{p}_{h}^{\MM,t}(\cdot|s,a)}(\hat{V}_{h+1}^{t,\pi})}{H^2}\left(\frac{H^{2} \beta\left(n_{h}^{t}(s, a), \delta\right)}{n_{h}^{t}(s, a)} \wedge 1\right)}+15H^2\frac{\phi(n_h^t(s,a),\delta)}{n_h^t(s,a)}\nonumber\\
            &+
            \left(1+\frac{1}{H}\right)\sum_{s'\in \mathcal{S}}\hat{p}^{\MM,t}_h(s'|s,a)\hat{e}_{h+1}^{t,\pi}(s',\pi_{h+1}(s');r).\nonumber\\
            &\leq \sqrt{\frac{\beta^r(n_h^t(s,a))}{n_h^t(s,a)}} +  3\sqrt{\frac{\text{Var}_{\hat{p}_{h}^{\MM,t}(\cdot|s,a)}(\hat{V}_{h+1}^{t,\pi})}{H^2}\left(\frac{H^{2} \beta\left(n_{h}^{t}(s, a), \delta\right)}{n_{h}^{t}(s, a)} \wedge 1\right)}+15H^2\frac{\phi(n_h^t(s,a),\delta)}{n_h^t(s,a)}\nonumber\\
            &+
            \left(1+\frac{1}{H}\right)\sum_{s'\in \mathcal{S}}\hat{p}^{\MM,t}_h(s'|s,a)\hat{e}_{h+1}^{t,\pi}(s',\pi_{h+1}(s');r),
        \end{align}
        where the first inequality follows in \cite{menard2021fast} page 8 (the readers can also see there how the high probability event $\mathcal{E}$ is used in page 8 and Lemma 10 there), and the second comes from $\mathcal{E}^r$. Then, observe that $\beta^r(n) \leq \phi(n,\delta)$, and we have
        $$\sqrt{\frac{\phi(n_h^t(s,a),\delta)}{n_h^t(s,a)}} = \sqrt{\frac{1}{H^2}\frac{H^2\phi(n_h^t(s,a),\delta)}{n_h^t(s,a)}}\leq \sqrt{\frac{1}{H^2}\left(\frac{H^{2} \beta\left(n_{h}^{t}(s, a), \delta\right)}{n_{h}^{t}(s, a)} \wedge 1\right)}+\frac{1}{H}\frac{H^2\phi(n_h^t(s,a),\delta)}{n_h^t(s,a)},$$
        for the reason that $\sqrt{x}\leq x$ if $x \geq 1$. Therefore, we have
        \begin{equation}\label{ieq:upe_recur}
            \begin{aligned}
            \hat{e}_{h}^{t,\pi}(s,a;r)
            \leq&
            \left(\frac{1}{H}+3\sqrt{\frac{\text{Var}_{\hat{p}_{h}^{\MM,t}(\cdot|s,a)}(\hat{V}_{h+1}^{t,\pi})}{H^2}}\right)\sqrt{\left(\frac{H^{2} \beta\left(n_{h}^{t}(s, a), \delta\right)}{n_{h}^{t}(s, a)} \wedge 1\right)}+16H^2\frac{\phi(n_h^t(s,a),\delta)}{n_h^t(s,a)}\\
            &+
            \left(1+\frac{1}{H}\right)\sum_{s'\in \mathcal{S}}\hat{p}^{\MM,t}_h(s'|s,a)\hat{e}_{h+1}^{t,\pi}(s',\pi_{h+1}(s');r).
            \end{aligned}
        \end{equation}
        Then, we can recursively define an \textbf{upper bound} for $\hat{e}_{h}^{t,\pi}(s,a;r)$ based on \eqref{ieq:upe_recur} as follows:
        \begin{align*}
            Z_{h}^{t,\pi}(s,a;r)
            =&
            \min\left\{
            H, \left(\frac{1}{H}+3\sqrt{\frac{\text{Var}_{\hat{p}_{h}^{\MM,t}(\cdot|s,a)}(\hat{V}_{h+1}^{t,\pi})}{H^2}}\right)\sqrt{\left(\frac{H^{2} \beta\left(n_{h}^{t}(s, a), \delta\right)}{n_{h}^{t}(s, a)} \wedge 1\right)}\right.\\
            &\left.
            +16H^2\frac{\phi(n_h^t(s,a),\delta)}{n_h^t(s,a)}+\left(1+\frac{1}{H}\right)\sum_{s'\in \mathcal{S}}\hat{p}^{\MM,t}_h(s'|s,a)Z_{h+1}^{t,\pi}(s',\pi_{h+1}(s');r)
            \right\}.
        \end{align*}
        Then, consider the following two sequences
        \begin{align*}
            \label{def:wy}
            Y_{h}^{t,\pi}(s,a;r)
            =&
            \left(\frac{1}{H}+3\sqrt{\frac{\text{Var}_{\hat{p}_{h}^{\MM,t}(\cdot|s,a)}(\hat{V}_{h+1}^{t,\pi})}{H^2}}\right)\sqrt{\left(\frac{H^{2} \beta\left(n_{h}^{t}(s, a), \delta\right)}{n_{h}^{t}(s, a)} \wedge 1\right)}\\
            &+
            \left(1+\frac{1}{H}\right)\sum_{s'\in \mathcal{S}}\hat{p}^{\MM,t}_h(s'|s,a)Y_{h+1}^{t,\pi}(s',\pi_{h+1}(s');r),\\
            W_{h}^{t,\pi}(s,a;r)
            =&
            \min\left\{
            H, 16H^2\frac{\phi(n_h^t(s,a),\delta)}{n_h^t(s,a)}
            \left(1+\frac{1}{H}\right)\sum_{s'\in \mathcal{S}}\hat{p}^{\MM,t}_h(s'|s,a)Z_{h+1}^{t,\pi}(s',\pi_{h+1}(s');r)
            \right\}.
        \end{align*}
We can prove by induction that for all $h,s,a$, 
$$
\hat{e}_{h}^{t,\pi}(s,a;r) \leq Z_{h}^{t, \pi}(s, a ; r) \leq Y_{h}^{t, \pi}(s, a ; r)+W_{h}^{t, \pi}(s, a).
$$
Therefore, to bound $\hat{e}_{1}^{t,\pi}(s_1,\pi(s_1);r)$, it suffices to bound $Y_{h}^{t, \pi}(s_1, \pi(s_1) ; r)+W_{h}^{t, \pi}(s_1, \pi(s_1))$. Denote $\hat{p}_h^{\MM,t,\pi}(s,a)$ the probability that the pair $(s,a)$ is visited under the estimated transition kernel following policy $\pi$, we have 

\begin{equation*}
    \begin{aligned}
&\hspace{5mm}Y_{1}^{t, \pi}\left(s_{1},\pi(s_1) ; r\right)\\ & = \sum_{s, a} \sum_{h=1}^{H} \widehat{p}_{h}^{\MM, t, \pi}(s, a)\left(1+\frac{1}{H}\right)^{h-1} \left(\frac{1}{H}+3\sqrt{\frac{\text{Var}_{\hat{p}_{h}^{\MM,t}(\cdot|s,a)}(\hat{V}_{h+1}^{t,\pi})}{H^2}}\right)\sqrt{\left(\frac{H^{2} \beta\left(n_{h}^{t}(s, a), \delta\right)}{n_{h}^{t}(s, a)} \wedge 1\right)}\\
& \leq 3e \sqrt{\sum_{s, a} \sum_{h=1}^{H} \widehat{p}_{h}^{\MM, t, \pi}(s, a) \frac{\text{Var}_{\hat{p}_{h}^{\MM,t}(\cdot|s,a)}(\hat{V}_{h+1}^{t,\pi})}{H^2}} \sqrt{\sum_{s, a} \widehat{p}_{h=1}^{t, \pi}(s, a)\left(\frac{H^{2} \beta\left(n_{h}^{t}(s, a), \delta\right)}{n_{h}^{t}(s, a)} \wedge 1\right)} \\
&\hspace{4mm} + e\sqrt{\sum_{s, a} \sum_{h=1}^{H} \widehat{p}_{h}^{\MM, t, \pi}(s, a) \frac{1}{H^2}} \sqrt{\sum_{s, a} \widehat{p}_{h=1}^{t, \pi}(s, a)\left(\frac{H^{2} \beta\left(n_{h}^{t}(s, a), \delta\right)}{n_{h}^{t}(s, a)} \wedge 1\right)}\\
    \end{aligned}
\end{equation*}

\begin{equation*}
    \begin{aligned}
& \leq \left(3 e \sqrt{\frac{1}{H^{2}} \mathbb{E}_{\pi, \widehat{p}_{h}^{\MM,t}}\left[\left(\sum_{h=1}^{H} r_{h}\left(s_{h}, a_{h}\right)-\widehat{V}_{1}^{\pi}\left(s_{1} ; r\right)\right)^2\right]}+\frac{e}{\sqrt{H}}\right) \sqrt{\sum_{s, a} \sum_{h=1}^{H} \widehat{p}_{h}^{\MM, t, \pi}(s, a)\left(\frac{H^{2} \beta\left(n_{h}^{t}(s, a), \delta\right)}{n_{h}^{t}(s, a)} \wedge 1\right)} \\
& \leq 4e \sqrt{\sum_{s, a} \sum_{h=1}^{H} \widehat{p}_{h}^{\MM,t,\pi}(s, a)\left(\frac{H^{2} \beta\left(n_{h}^{t}(s, a), \delta\right)}{n_{h}^{t}(s, a)} \wedge 1\right)} \leq 4 e \sqrt{W_{1}^{t, \pi}\left(s_{1},\pi(s_1)\right)},
    \end{aligned}
\end{equation*}
Where the second inequality comes from the law of total variance (\cite{menard2021fast} Lemma 7), and the last inequality comes from page 24 (Step 3) in \cite{menard2021fast}. Therefore, we have 
    \[
        \hat{e}_{h}^{t,\pi}(s_1,\pi(s_1);r)\leq 4e\sqrt{\max_{a\in\mathcal{A}}W_{1}^{t}(s_1,a)}+\max_{a\in\mathcal{A}}W_{1}^{t}(s_1,a).
    \]

\end{proof}
\begin{proof}[Proof of Lemma \ref{lem_tau1}]
 We first provide an upper bound on $W_{h}^{t}(s, a)$ for all $(s, a, h)$ and $t$. By definition \eqref{eqn_W}, if $n_{h}^{t}(s, a)>0$, we have
$$
\begin{aligned}
& W_{h}^{t}(s, a) \leq 16 H^{2} \frac{\beta\left(n_{h}^{t}(s, a), \delta\right)}{n_{h}^{t}(s, a)}+\left(1+\frac{1}{H}\right) \sum_{s^{\prime}} \widehat{p}_{h}^{\MM,t}\left(s^{\prime} \mid s, a\right) \max _{a^{\prime}} W_{h+1}^{t}\left(s^{\prime}, a^{\prime}\right) \\
&=16 H^{2} \frac{\beta\left(n_{h}^{t}(s, a), \delta\right)}{n_{h}^{t}(s, a)}+\left(1+\frac{1}{H}\right)\sum_{s'\in\mathcal{S}}\left(\widehat{p}_{h}^{\MM,t}(s'|s,a)-p^{\MM}_{h}(s'|s,a)\right) W_{h+1}^{t}(s', \pi^{t+1}_{h+1}(s')) \\
&+\left(1+\frac{1}{H}\right)\sum_{s'\in\mathcal{S}} p^{\MM}_{h}(s'|s,a)W_{h+1}^{t}(s', \pi_{h+1}^{t+1}(s'))
\end{aligned}
$$
From Lemma 10 in \cite{menard2021fast} and the Bernstein inequality we get (see page 25 in \cite{menard2021fast} for more details)
$$
W_{h}^{t}(s, a) \leq 22 H^{2}\left(\frac{\beta\left(n_{h}^{t}(s, a), \delta\right)}{n_{h}^{t}(s, a)} \wedge 1\right)+\left(1+\frac{3}{H}\right) \sum_{s'\in \mathcal{S}}p_{h}(s'|s,a)  W_{h+1}^{t}(s, \pi_{h+1}^{t+1}(s))
$$
Unfolding the above equation and using $(1+3/H)^H\leq e^3$ we have
$$
W_{1}^{t}\left(s_1,\pi_{1}^{t+1}(s_1)\right) \leq 22 e^{3} H^{2} \sum_{h=1}^{H} \sum_{s, a} p_{h}^{\MM,t+1}(s, a)\left(\frac{\beta\left(n_{h}^{t}(s, a), \delta\right)}{n_{h}^{t}(s, a)} \wedge 1\right)
$$
In this proof, we choose the high probability event to be $\mathcal{E}^{cnt}$, under which we have (see \cite{menard2021fast} lemma 8)
\begin{align}\label{eqn_w1bound}
    W_{1}^{t}\left(s_1,\pi_{1}^{t+1}(s_1)\right) \leq 88 e^{3} H^{2} \sum_{h=1}^{H} \sum_{s, a} p_{h}^{\MM,t+1}(s, a) \frac{\beta\left(\bar{n}_{h}^{t}(s, a), \delta\right)}{\bar{n}_{h}^{t}(s, a) \vee 1},
\end{align}
where we recall that $\bar{n}_{h}^{t}(s, a) = \sum_{l=1}^{t} p_{h}^{\MM,\pi^l}(s, a)$ is the pseudo-count.

Next, we are going to sum the above inequality over $t \leq T$ for $T<\tau$. Due to the stopping rule, we have 
$$
\varepsilon \leq 4e \sqrt{ W_{1}^{t}\left(s_{1},\pi_{1}^{t+1}(s_1)\right)}+ W_{1}^{t}\left(s_{1},\pi_{1}^{t+1}(s_1)\right).
$$
Summing the over the above inequalities for $0 \leq t \leq T$, followed by Cauchy-Schwarz inequality, we have
$$
\begin{aligned}
(T+1) \varepsilon & \leq \sum_{t=0}^{T}\left( 4e \sqrt{ W_{1}^{t}\left(s_{1},\pi_{1}^{t+1}(s_1)\right)}+ W_{1}^{t}\left(s_{1},\pi_{1}^{t+1}(s_1)\right)\right) \\
& \leq 4e \sqrt{(T+1) \sum_{t=0}^{T} W_{1}^{t}\left(s_{1},\pi_{1}^{t+1}(s_1)\right)}+\sum_{t=0}^{T} W_{1}^{t}\left(s_{1},\pi_{1}^{t+1}(s_1)\right) .
\end{aligned}
$$
Next, from \eqref{eqn_w1bound}, the property that $\phi(\cdot, \delta)$ is increasing and Lemma 9 in \cite{menard2021fast}, we have
$$
\begin{aligned}
\sum_{t=0}^{T} W_{1}^{t}\left(s_{1},\pi_{1}^{t+1}(s_1)\right) & \leq 88 e^{3} H^{2} \sum_{t=0}^{T} \sum_{h=1}^{H} \sum_{s, a} p_{h}^{\MM,t+1}(s, a) \frac{\beta\left(\bar{n}_{h}^{t}(s, a), \delta\right)}{\bar{n}_{h}^{t}(s, a) \vee 1} \\
& \leq 88 e^{3} H^{2} \phi(T, \delta) \sum_{t=0}^{T} \sum_{h=1}^{H} \sum_{s, a} p_{h}^{\MM,t+1}(s, a) \frac{1}{\bar{n}_{h}^{t}(s, a) \vee 1} \\
& =88 e^{3} H^{2} \phi(T, \delta) \sum_{h=1}^{H} \sum_{s, a} \sum_{t=0}^{T} \frac{\bar{n}_{h}^{t+1}(s, a)-\bar{n}_{h}^{t}(s, a)}{\bar{n}_{h}^{t}(s, a) \vee 1} \\
& \leq 352 e^{3} H^{3} S A \log (T+2) \phi(T, \delta)
\end{aligned}
$$
Therefore, we have
$$
(T+1) \varepsilon \leq 76 e^{3} \sqrt{(T+1) H^{3} S A \log (T+2) \phi(T, \delta)}+352 e^{3} H^{3} S A \log (T+2) \phi(T, \delta)
$$
Lastly, using Lemma 13 in \cite{menard2021fast} we have 
$$
\tau \leq \frac{H^{3} S A}{\varepsilon^{2}}(\log (4 S A H / \delta)+S) C_{1}+1
$$
where $C_{1} = 9000e^6\log^2\left(e^{18}\left(\log(4HSAT/\delta)+S\right)\frac{H^4SA}{\epsilon}\right)$.

\end{proof}

\begin{proof}[Proof of Lemma \ref{lem:hp}]
    To prove the probability bound for the first two sets $\mathcal{E}$ and $\mathcal{E}^{cnt}$, we refer to Lemma 3 in \cite{menard2021fast}. To show the probability bound for $\mathcal{E}^r$, by Hoeffding's inequality, we have with probability at least $1-\frac{\delta}{4^{11}e^{12}S^{12}A^{12}H^{12}}$
    \begin{align}
        \label{ieq:rbdd}
        |r_h^{\MM}(s,a)-\hat{r}_h^{\MM}(s,a)|\leq\frac{\beta^{r}\left(n_{h}^{t}(s, a)\right)}{\sqrt{n_h^t(s,a)}}        
    \end{align}
    for any fixed episode $t$, step $h$, and state-action pair $(s,a)\in\mathcal{S}\times\mathcal{A}$. Then, taking a union bound we have \eqref{ieq:rbdd} holds for all state-action pairs, and $t\leq 4^{10}e^{12} S^{11}A^{11}H^{11}/(\epsilon^{12}\delta^{12})$ with probability no less than $1-\frac{\delta}{3}$.

    We remark here that although we have this episode's upper bound, our algorithm will stop before reaching this upper bound. Specifically, by calculations, we have for all $H,A,S\geq1$
    \begin{align*}
        \frac{4^{10} e^{12} S^{11}A^{11}H^{11}}{\epsilon^{12}\delta^{12}}
        &\geq
        1048576e^{12} \frac{S^{11}A^{11}H^{11}}{\epsilon^{12}\delta^{12}}\\
        &\geq
        \left(36\cdot 10^4e^6+32\cdot10^4e^8\right)\frac{H^7S^4A^4}{\epsilon^4\delta^2}\\
        &\geq
        C_1\frac{H^3SA}{\epsilon^2}\cdot\frac{4SAH}{\delta},
    \end{align*}
    which is larger than the bound in Theorem \ref{thm_RFE}. Thus, our choice of $\beta^r$ guarantees that with high probability, the algorithm will find an $\epsilon$-optimal solution before reaching the maximum number of episodes.
\end{proof}

\section{Supplementary Materials for Experiment}
\label{apnd:B}
\subsection{Flappy Bird Experiment}
\paragraph{More details for Figure \ref{fig:regret_bird}.}\

 We implement three algorithms in Figure \ref{fig:regret_bird}. For each algorithm, due to the computation efficiency, the policy will be updated every 5000 episode in Figure \ref{fig:regret_bird_greedy} (Policy Greedy) and 1000 episode in  Figure \ref{fig:regret_bird_safe} (Policy Safe), and then apply the updated policies for the next 1000/5000 episode respectively. All three algorithms are trained by $8\times 10^5$ for Policy Greedy and $2\times 10^5$ for Policy Safe. All experiments are repeated $5$ times with the mean of results being reported.  More details about implementations are as follows:
\begin{itemize}
    \item RFE-ADvice (FRE-AD): Based on the estimated transition kernel and reward from Algorithm \ref{alg:alg2}. At episode $t$, the algorithm uses planning oracles to obtain the optimal policy $\hat{\pi}^{t}$ for the MDP $\{\mathcal{S},\bar{\mathcal{A}}, H, \hat{p}^{\MM,t}, r\}$. Note the rewards are known, we replace the reward's estimators $\hat{r}$ used in Algorithm \ref{alg:alg2} by the true reward $r$ of the environment. For the parameter of the algorithm, we set $\delta=0.1$ and we do not set the convergence checking step so no need to input $\epsilon$. We also use $0.1\cdot \beta$ instead of $\beta$ to control the "exploration" bonus (note $\beta$ function measures the uncertainty of the current state). In Figure \ref{fig:regret_bird_greedy}, \ref{fig:regret_bird_safe}, we evaluate the value gap between the (estimated) optimal policy $\hat{\pi}^{t}$  and the optimal policy in every $1000$ episode and use the cumulative sum of $1000\times \text{value gap}$ as the "regret" of RFE-AD. 
    \item UCB-ADherence (UCB-AD): This algorithm is the Algorithm \ref{alg:alg1} that outputs policy at different episode $t$ based on the upper confidence bound of the adherence level $\theta$. When implementing, we use $$C(\theta,n,T,\delta)=0.4*\sqrt{\frac{2\log(n)}{n}},$$
    as the function to measure uncertainty. Three subfigures of Figure \ref{fig:regret_bird} show the (cumulative) regret of UCB-AD. To test the robustness of UCB-AD, Figure \ref{fig:regret_UCB_AD_thetas} "zooms in" its regret under different values of $\theta$ for both policies.  
    \item EULER: The algorithm is mainly based on \cite{zanette2019tighter} with two differences. (1) Since the reward is known, we replace all estimates of reward $r$ by their true values (and set the uncertainty bonus as $0$) in the algorithm; (2) The algorithm will fail when the unreachable states, where can not be arrived with probability $1$ in the MDP, are not explicitly revealed. Since such states can not be known because of the unknown policy of humans, we feed additional
     $3\times 10^5$ episodes for both policies as the exploration period (and thus the total number of training episodes is $11\times 10^5$ and $5\times 10^5$ for EULER), after which we set the states with $0$ observation/count as unreachable states. We further select the parameter $\delta=0.1$ used in EULER. In Figure \ref{fig:regret_bird}, we show the (cumulative) regret only after the exploration period ($3\times 10^5$ episode).   
\end{itemize}

\paragraph{More details for Figure \ref{fig:cvg_beta}.}\

We implement RFE-$\beta$ in Figure \ref{fig:cvg_beta}, which is based on the estimated transition kernel and reward from Algorithm \ref{alg:alg2} at episode $t$, the algorithm uses planning oracles to obtain the optimal policy $\{\mathcal{S},\bar{\mathcal{A}}, H, \hat{p}^{\MM,t}, r_{\beta}\}$ with $\beta\in\{0,0.2,0.4\}$, and all the other settings are same as Figure \ref{fig:regret_bird}. Figure \ref{fig_apnd:cvg_beta} shows the value gaps of RFE-$\beta$, where for completeness we also include the copy of \ref{fig:cvg_beta}.

\begin{figure}[ht!]
\centering
\begin{subfigure}{0.49\textwidth}
\includegraphics[height=0.22\textheight]{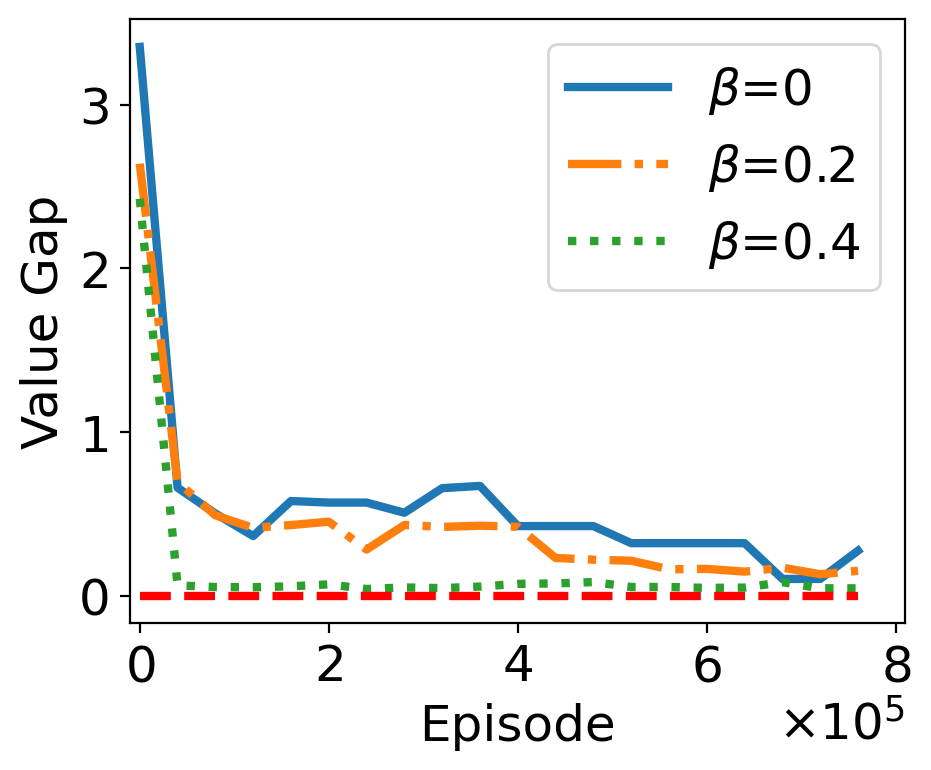}
   \caption{Policy Greedy.}
\end{subfigure}
\begin{subfigure}{0.49\textwidth}
\includegraphics[height=0.22\textheight]{pics/reward_beta_safe.png}
   \caption{Policy Safe.}
\end{subfigure}
\caption{Value Gaps of RFE-$\beta$.}
\label{fig_apnd:cvg_beta}
\end{figure}

\paragraph{More details for Figure \ref{fig:cvg_CMDP_FR}, \ref{fig:value_post_CMDP}.}\

\textbf{Environment.} We change our environment to illustrate the challenge with advice budget, which is shown in Figure \ref{fig_apnd:bird_CMDP_env}. We should note to achieve the plotted advice policy, which is optimal without advice budget constraint for Policy Greedy, we should at least advise twice to make the bird go through the wall for two stars instead of getting the star at the beginning and hitting the wall. Thus, setting the advice budget as $1$ will make both the policy itself and the learning process harder than with enough advice budget.

\textbf{Algorithm Implementation.} We implement two algorithms for this environment. For each algorithm, due to the computation efficiency, the policy will be updated every $50$ episodes, and then apply the updated policies for the next $50$ episode. For Figure \ref{fig:cvg_CMDP_FR} and \ref{fig_apnd:CMDP_budget_cvg}, we evaluate both algorithms for every $50$ episodes to compute the value gap and advice count gap. Both algorithms are trained by $1500$ episode and all experiments are repeated $5$ times with the mean of the results being reported. Details about implementations are as follows:
\begin{itemize}
    \item  RFE-CMDP: The algorithm is based on the estimated transition kernel and reward from Algorithm \ref{alg:alg2} at episode $t$ and it uses planning oracles to obtain the optimal policy $\hat{\pi}^{t}_{D}$ for the CMDP \eqref{eqn_sample_cmdp}. For the algorithm's parameters, we set them as the same values as in \ref{fig:regret_bird}.
    \item UC-CFH: The algorithm is based on \cite{kalagarla2021sample}. We set its parameter with $\epsilon=1$ and $\delta=0.5$ in this experiment.
\end{itemize}

\begin{figure}[ht!]
\centering
\begin{subfigure}[b]{0.31\textwidth}
\centering
\includegraphics[height=0.12\textheight]{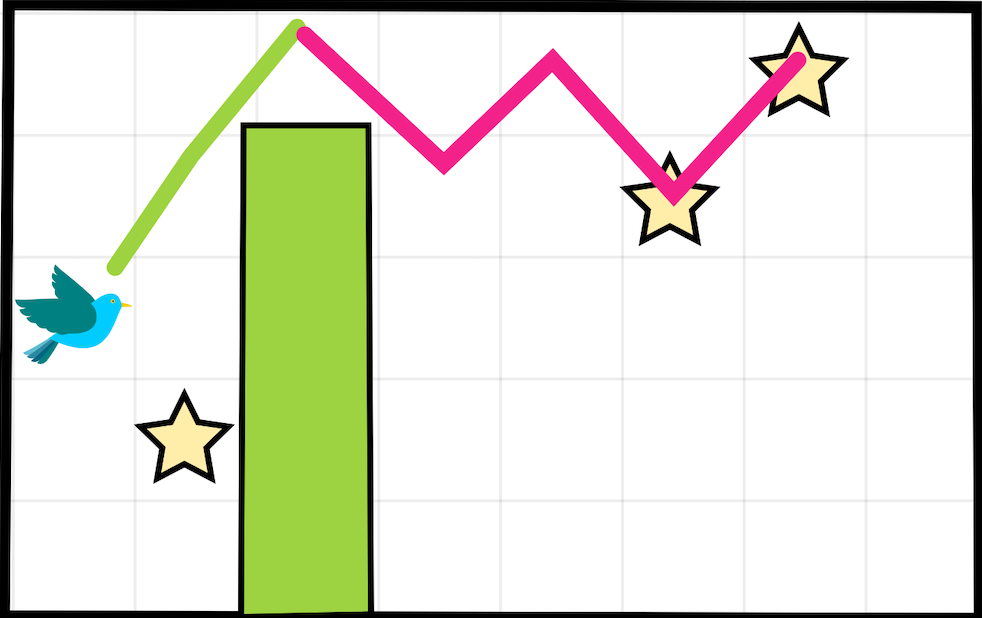}
   \caption{Environment and optimal trajectory without advice budget.}
   \label{fig_apnd:bird_CMDP_env}
\end{subfigure}
\hfill 
\begin{subfigure}[b]{0.31\textwidth}
\centering
\includegraphics[height=0.15\textheight]{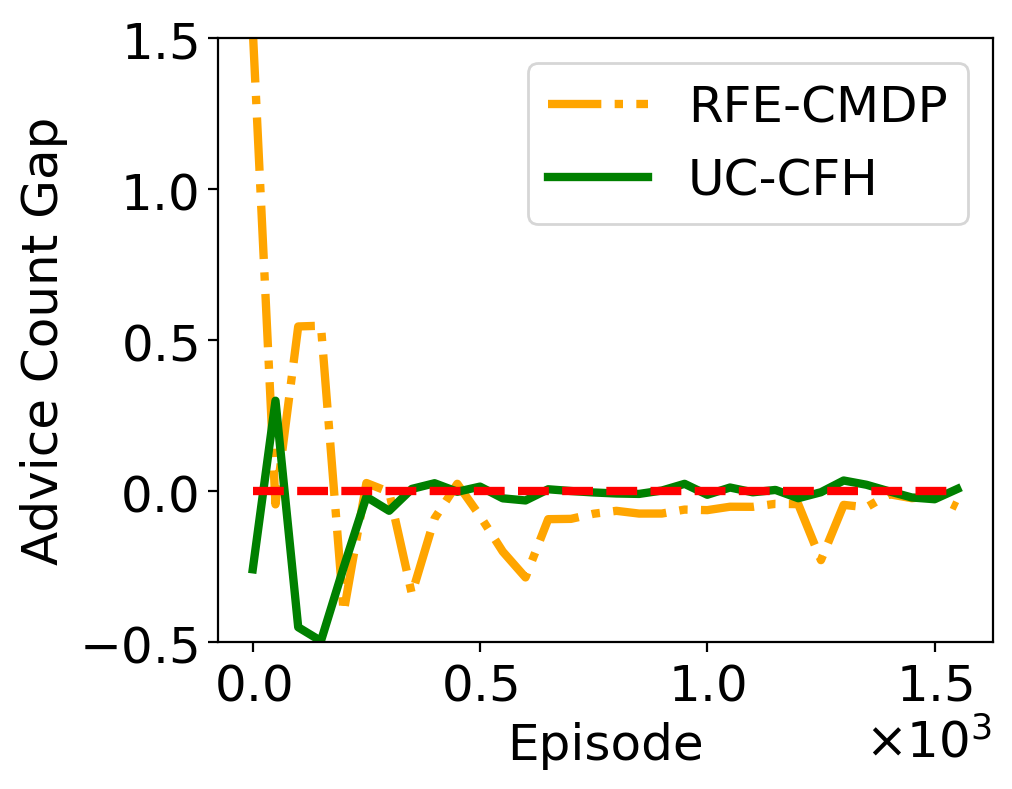}
   \caption{Advice count gaps w.r.t. episode.}
   \label{fig_apnd:CMDP_budget_cvg}
\end{subfigure}
\hfill 
\begin{subfigure}[b]{0.31\textwidth}
\centering
\includegraphics[height=0.15\textheight]{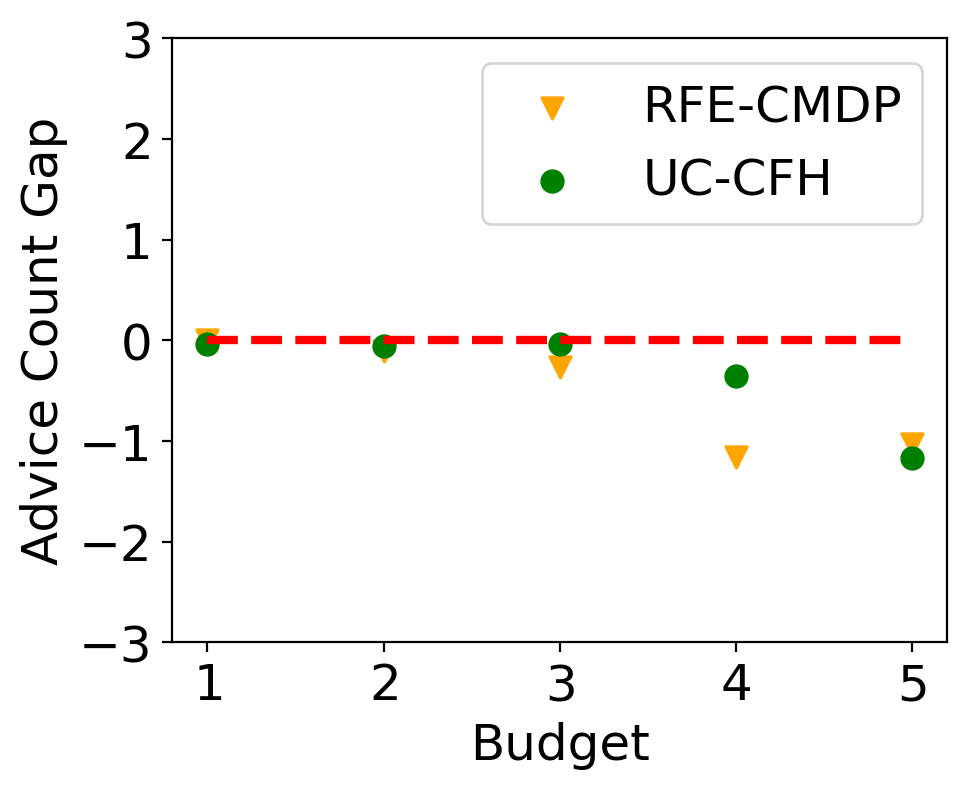}
   \caption{Advice count gaps when evaluating.}
\end{subfigure}
\caption{Additional results for RFE-CMDP. (a) shows the environment for testing the algorithms and one trajectory of optimal advice policy with advising twice: the red line means the machine defers and the green line means the machine advises. (b) and (c) show the budget violations through the advice count gap, which is computed by using the advice budget minus expected advising times of the policy, w.r.t. training episode and when evaluating respectively.}
\label{fig_apnd:bird_CMDP}
\end{figure}

\subsection{Car Driving Environment}
In our car driving environment, we have three lanes and a horizon of 10. Each cell within the environment can either be empty, contain a stone, or have a car present. The types of cells follow independent and identically distributed (i.i.d.) distributions, which will be specified later. The objective for the driver is to maximize the distance covered by the car. However, colliding with another car or reaching the boundaries results in the destruction of the car, terminating the episode. Encountering a stone can also cause minor damage to the car, but the car will still be operating. The driver's goal is to drive on the empty road and avoid any obstacles. The experiment is adapted from \cite{meresht2020learning}.

\paragraph{Car Driving MDP.}\ 

We assume the machine is myopic up to two rows: It can observe the next two rows at most. Thus, the state space can be defined by the 9 cells' types in the current three rows with a total of $3^9=19683$ states. The environment can be represented by Figure \ref{fig:car_env}.

The action space is $\mathcal{A}=\{\text{Left, Straight, Right}\}$. The car will always keep moving to the next row unless it hits another car or the boundary. Further, the ``Left'' action will move the car to the left of the current lane, the ``Right'' action will move to the right, and the ``Straight'' action will keep the car in its current lane.

The car will get a reward of $1$ when the current cell where it is located is empty, a reward of $0.5$ when it has a stone, and a reward $0$ when it has a car or is out of the boundary (also, the environment will be terminated). The cell's type is sampled by $(0.4,0.3,0.3)$ for the empty, stone, and car respectively. 

\paragraph{Human behavior: $\pi^{\HH}$ and $\theta$.}\

We consider a myopic driver who is only aware of other cars in the next row. Therefore, the driver's policy is to avoid the cars in the next row, and if there are multiple equivalent actions it will select them randomly with an equal probability.
We set the adherence level $\theta$ as the follows
\begin{equation*}
    \theta(a,s)=\begin{cases}
        0.9 &\quad \text{if $a=$\text{Straight},}\\
        0.7 &\quad \text{otherwise.}
    \end{cases}
\end{equation*}
Thus, the driver will adhere to the advice with probability $0.9$ if the advice is ``Straight'' and otherwise $0.7$, and the intuition is that the driver wants to avoid changing the lane too often.

\paragraph{Experiment setting and results.}\

We train RFE-AD and UCB-AD for the \textit{Car Driving} environment, with $4\times10^5$ training episodes and the parameters same as the Flappy Bird environment.  We note that because the environment is $\mathcal{E}_1$, UCB-AD is trained with the knowledge of the distribution of the cell's type. Due to the large state space of this environment, it makes any planning algorithm computationally intensive. Thus we update and evaluate the policy for RFE-AD and UCB-AD every $8000$ episode. The algorithms' parameters are the same as Flappy Bird environment. All experiments
are repeated $5$ times with the mean of results being reported.

As demonstrated in Figure \ref{fig:regret_car}, UCB-AD not only outperforms RFE-AD but also achieves a near-optimal policy at a very early stage. We believe the strong performance of UCB-AD comes from the monotone property of Proposition \ref{prop_monotone} and that the optimistic policy happens to be the optimal policy in this setting. In Figure \ref{fig:beta_car}, we further evaluate the RFE-$\beta$ algorithm for different $\beta$'s based on the estimated transition kernel from \ref{fig:regret_car} for $\beta=0.1,\ldots,1$, and find $\{V_{\beta}^{\hat{\pi}_{\beta}}\}_{\beta >0}$ close to $\{V_{\beta}^*\}_{\beta>0}$ in a consistent manner.

\begin{figure}[ht!]
\centering
    \begin{subfigure}[b]{0.31\textwidth}
    \hspace{1.3cm}
\includegraphics[height=0.16\textheight]{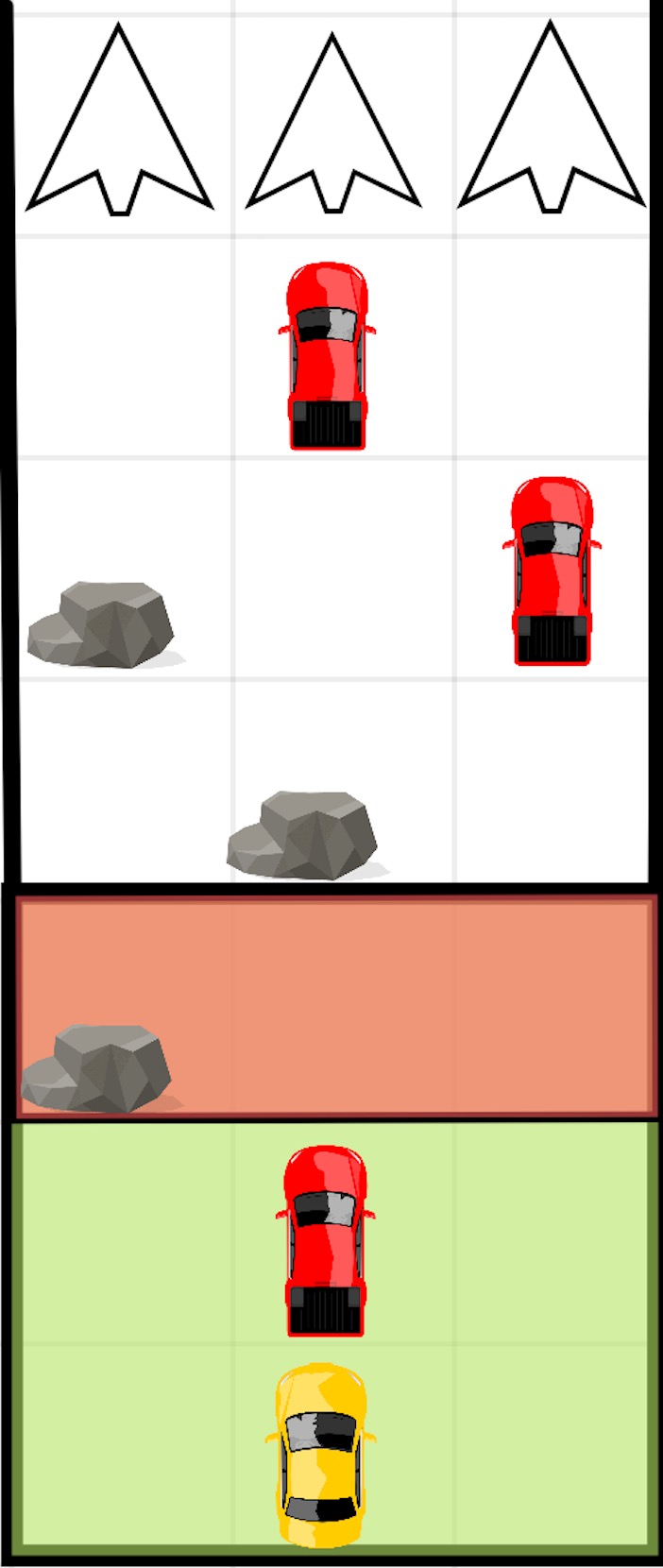}
    \caption{Car Driving Environment.}
    \label{fig:car_env}
    \end{subfigure}
\begin{subfigure}[b]{0.31\textwidth}
\includegraphics[height=0.16\textheight]{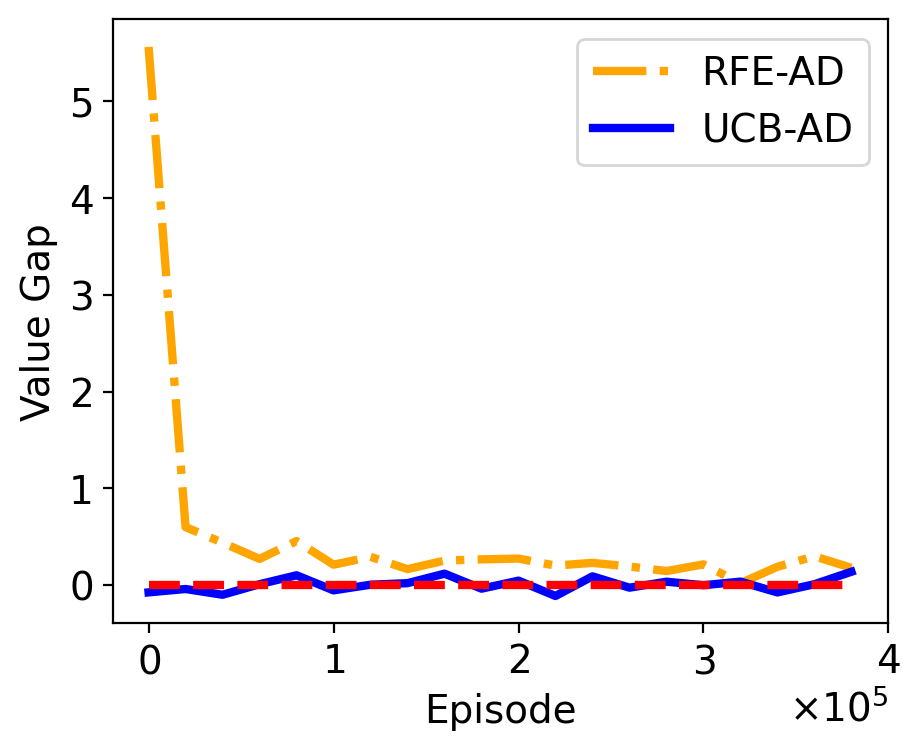}
   \caption{Value gaps w.r.t. episode.}
   \label{fig:regret_car} 
\end{subfigure}
\begin{subfigure}[b]{0.31\textwidth}
\includegraphics[height=0.16\textheight]{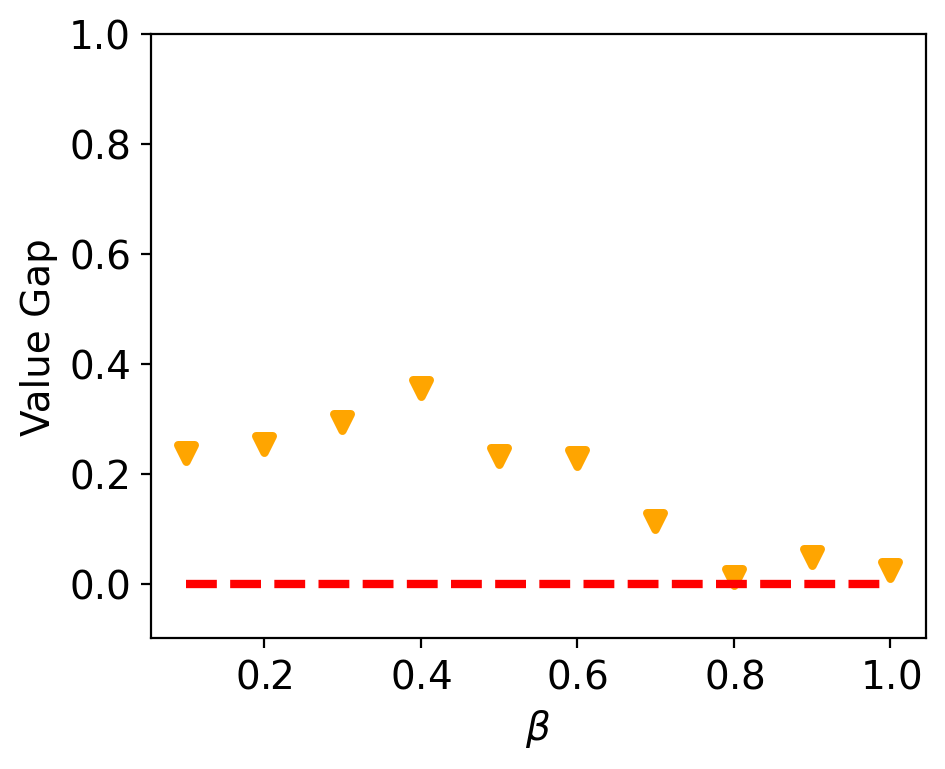}
   \caption{Value gaps when evaluating.}
   \label{fig:beta_car} 
\end{subfigure}
\caption{Environment and value gaps of algorithms. Car Driving environment (a), the driver needs to move forward as far as possible. The green cells are observable to the driver and the red cells (with the green cells) are observable to the machine.  For (b) and (c), the value gap is defined as the difference between optimal value $V^*$ (or $V^*_{\beta}$) and $V^{\hat{\pi}}$ (or $V^{\hat{\pi}_{\beta}}_{\beta}$), with the red dashed line as the benchmark for $0$ loss of the policy.  (b) shows the value gaps of RFE-AD and UCB-AD with respect to the training episode. (c) shows the value gaps when evaluating RFE-$\beta$ with $\beta=\{0.1,0.2,...,1\}$. }
\label{fig:car}
\end{figure}

\end{document}